\documentclass[sigconf]{acmart}

\usepackage[linesnumbered,ruled,vlined]{algorithm2e}
\usepackage{bm}
\usepackage{balance}

\begin{document}
\title{Reinforcement Learning to Rank in E-Commerce Search Engine: Formalization, Analysis, and Application}

\author{Yujing Hu}
\affiliation{\institution{Alibaba Group} 
\state{Hangzhou}
\country{China}}
\email{yujing.hyj@alibaba-inc.com}

\author{Qing Da}
\affiliation{\institution{Alibaba Group}
\state{Hangzhou}
\country{China}}
\email{daqing.dq@alibaba-inc.com}

\author{Anxiang Zeng}
\affiliation{\institution{Alibaba Group}
\state{Hangzhou}
\country{China}}
\email{renzhong@taobao.com}

\author{Yang Yu}
\affiliation{
\institution{National Key Laboratory for \\ Novel Software Technology, \\ Nanjing University}
\state{Nanjing}
\country{China}}
\email{yuy@nju.edu.cn}

\author{Yinghui Xu}
\affiliation{\institution{Artificial Intelligence Department, Zhejiang Cainiao Supply Chain Management Co., Ltd.}
\state{Hangzhou}
\country{China}}
\email{renji.xyh@taobao.com}

\copyrightyear{2018} 
\acmYear{2018} 
\setcopyright{acmcopyright}
\acmConference[KDD '18]{The 24th ACM SIGKDD International Conference on Knowledge Discovery \& Data Mining}{August 19--23, 2018}{London, United Kingdom}
\acmPrice{15.00}
\acmDOI{10.1145/3219819.3219846}
\acmISBN{978-1-4503-5552-0/18/08}
\settopmatter{printacmref=true} 
\fancyhead{} 

\begin{abstract}
In E-commerce platforms such as \textit{Amazon} and \textit{TaoBao}, ranking items in a search session is a typical multi-step decision-making problem. Learning to rank (LTR) methods have been widely applied to ranking problems. However, such methods often consider different ranking steps in a session to be independent, which conversely may be highly correlated to each other. For better utilizing the correlation between different ranking steps, in this paper, we propose to use reinforcement learning (RL) to learn an optimal ranking policy which maximizes the expected accumulative rewards in a search session. Firstly, we formally define the concept of search session Markov decision process (SSMDP) to formulate the multi-step ranking problem. Secondly, we analyze the property of SSMDP and theoretically prove the necessity of maximizing accumulative rewards. Lastly, we propose a novel policy gradient algorithm for learning an optimal ranking policy, which is able to deal with the problem of high reward variance and unbalanced reward distribution of an SSMDP. Experiments are conducted in simulation and \textit{TaoBao} search engine. The results demonstrate that our algorithm performs much better than the state-of-the-art LTR methods, with more than $40\%$ and $30\%$ growth of total transaction amount in the simulation and the real application, respectively.
\end{abstract}

%
%



\maketitle

\section{Introduction}
Over past decades, shopping online has become an important part of people's daily life, requiring the E-commerce giants like \textit{Amazon}, \textit{eBay} and \textit{TaoBao} to provide stable and fascinating services for hundreds of millions of users all over the world. Among these services, commodity search is the fundamental infrastructure of these E-commerce platforms, affording users the opportunities to search commodities, browse product information and make comparisons. For example, every day millions of users choose to purchase commodities through \textit{TaoBao} search engine.

In this paper, we focus on the problem of ranking items in large-scale item search engines, which refers to assigning each item a score and sorting the items according to their scores. Generally, a search session between a user and the search engine is a multi-step ranking problem as follows:
\begin{enumerate}
\item the user inputs a query in the blank of the search engine, 
\item the search engine ranks the items related to the query and displays the top $K$ items (e.g., $K = 10$) in a page,
\item  the user makes some operations (e.g., click items, buy some certain item or just request a new page of the same query) on the page, 
\item when a new page is requested, the search engine reranks the rest of the items and display the top $K$ items.
\end{enumerate}
These four steps will repeat until the user buys some items or just leaves the search session. Empirically, a successful transaction always involves multiple rounds of the above process.

The operations of users in a search session may indicate their personal intentions and preference on items. From a statistical view, these signals can be utilized to learn a ranking function which satisfies the users' demand. This motivates the marriage of machine learning and information retrieval, namely the learning to rank (LTR) methods \cite{joachims2002optimizing,liu2009learning}, which learns a ranking function by classification or regression from training data. The major paradigms of supervised LTR methods are pointwise \cite{nallapati2004discriminative,li2008mcrank}, pairwise \cite{cao2006adapting,burges2005learning}, and listwise \cite{cao2007learning}. Recently, online learning techniques such as regret minimization \cite{auer2002using,langford2008epoch,kveton2015cascading} have been introduced into the LTR domain for directly learning from user signals. Compared with offline LTR, online LTR avoids the mismatch between manually curated labels, user intent \cite{yue2009interactively} and the expensive cost of creating labeled data sets. Although rigorous mathematical models are adopted for problem formalization \cite{yue2009interactively,kveton2015cascading,zoghi2017online} and guarantees on regret bounds are established, most of those works only consider a one-shot ranking problem, which means that the interaction between the search engine and each user contains only one round of ranking-and-feedback activity. However, in practice, a search session often contains multiple rounds of interactions and the sequential correlation between each round may be an important factor for ranking, which has not been well investigated.

In this paper, we consider the multi-step sequential ranking problem mentioned above and propose a novel reinforcement learning (RL) algorithm for learning an optimal ranking policy. The major contributions of this paper are as follows. 
\begin{itemize}
\item We formally define the concept of search session Markov decision process (SSMDP) to formulate the multi-step ranking problem, by identifying the state space, reward function and state transition function.
\item We theoretically prove that maximizing accumulative rewards is necessary, indicating that the different ranking steps in a session are tightly correlated rather than independent.
\item We propose a novel algorithm named deterministic policy gradient with full backup estimation (DPG-FBE), designed for the problem of high reward variance and unbalanced reward distribution of SSMDP, which could be hardly dealt with even for existing state-of-the-art RL algorithms.
\item We empirically demonstrate that our algorithm performs much better than online LTR methods, with more than $40\%$ and $30\%$ growth of total transaction amount in the simulation and the \textit{TaoBao} application, respectively.
\end{itemize}

The rest of the paper is organized as follows.  Section~\ref{Sec:2} introduces the background of this work. The problem description, analysis of SSMDP and the proposed algorithm are stated in Section ~\ref{Sec:3}, ~\ref{Sec:4}, ~\ref{Sec:5}, respectively. The experimental results are shown in Section ~\ref{Sec:6}, and Section~\ref{Sec:7} concludes the paper finally.

\section{Background} \label{Sec:2}
In this section, we briefly review some key concepts of reinforcement learning and the related work in the online LTR domain. We start from the reinforcement learning part.

\subsection{Reinforcement Learning} \label{SubSec:2.1}
Reinforcement learning (RL) \cite{sutton1998reinforcement} is a learning technique that an agent learns from the interactions between the environment by trial-and-error. The fundamental mathematical model of reinforcement learning is Markov decision process (MDP).

\begin{definition} [Markov Decision Process] \label{Def:MDP}
A Markov decision process is a tuple $\mathcal{M} = \langle \mathcal{S}, \mathcal{A}, \mathcal{R}, \mathcal{P}, \gamma \rangle $, where $\mathcal{S}$ is the state space, $\mathcal{A}$ is the action space of the agent, $\mathcal{R}: \mathcal{S} \times \mathcal{A} \times \mathcal{S} \rightarrow \mathbb{R}$ is the reward function, $\mathcal{P}: \mathcal{S} \times \mathcal{A} \times \mathcal{S} \rightarrow [0,1]$ is the state transition function and $\gamma \in [0,1]$ is the discount rate. 
\end{definition}

The objective of an agent in an MDP is to find an optimal policy which maximizes the expected accumulative rewards starting from any state $s$ (typically under the infinite-horizon discounted setting), which is defined by $V^*(s) = \max_{\pi} \mathbb{E}^{\pi} \big\{ \sum_{k=0}^{\infty} \gamma^{k} r_{t+k} \big| s_t = s \big\}$, where $\pi: \mathcal{S} \times \mathcal{A} \rightarrow [0,1]$ denotes any policy of the agent, $\mathbb{E}^{\pi}$ stands for expectation under policy $\pi$, $t$ is the current time step, $k$ is a future time step, and $r_{t+k}$ is the immediate reward at the time step $(t+k)$. This goal is equivalent to finding the optimal state-action value $Q^*(s,a) = \max_{\pi} \mathbb{E}^{\pi} \Big\{ \sum_{k=0}^{\infty} \gamma^{k} r_{t+k} \big| s_t = s, a_t = a \Big\}$ for any state-action pair $(s,a)$. In finite-horizon setting with a time horizon $T$, the objective of an agent can be reinterpreted as the finding the optimal policy which maximizes the expected $T$-step discounted return $\mathbb{E}^{\pi} \big\{ \sum_{k=0}^{T} \gamma^{k} r_{t+k} \big| s_t = s \big\}$ or undiscounted return $\mathbb{E}^{\pi} \big\{ \sum_{k=0}^{T} r_{t+k} \big| s_t = s \big\}$\footnote{The undiscounted return is a special case in discount setting with $\gamma=1$.} in the discounted and undiscounted reward cases, respectively.

An optimal policy can be found by computing the optimal state-value function $V^*$ or the optimal state-action value function $Q^*$. Early methods such as dynamic programming \cite{sutton1998reinforcement} and temporal-difference learning \cite{watkins1989learning} rely on a table to store and compute the value functions. However, such tabular methods cannot scale up in large-scale state/action space problems due to the curse of dimensionality. Function approximation is widely used to address the scalability issues of RL. By using a parameterized function (e.g., linear functions \cite{MaeiSBS10}, neural networks \cite{mnih2015human,silver2016mastering}) to represent the value function or the policy (a.k.a value function approximation and policy gradient method respectively), the learning problem is transformed to optimizing the function parameters according to reward signals.
In recent years, policy gradient methods \cite{sutton2000policy,silver2014deterministic,schulman2015trust} have drawn much attention in the RL domain. The explicit parameterized representation of policy enables the learning agent to directly search in the policy space and avoids the policy degradation problem of value function approximation. 

\begin{figure*}
  \centering
  \includegraphics[scale=0.4]{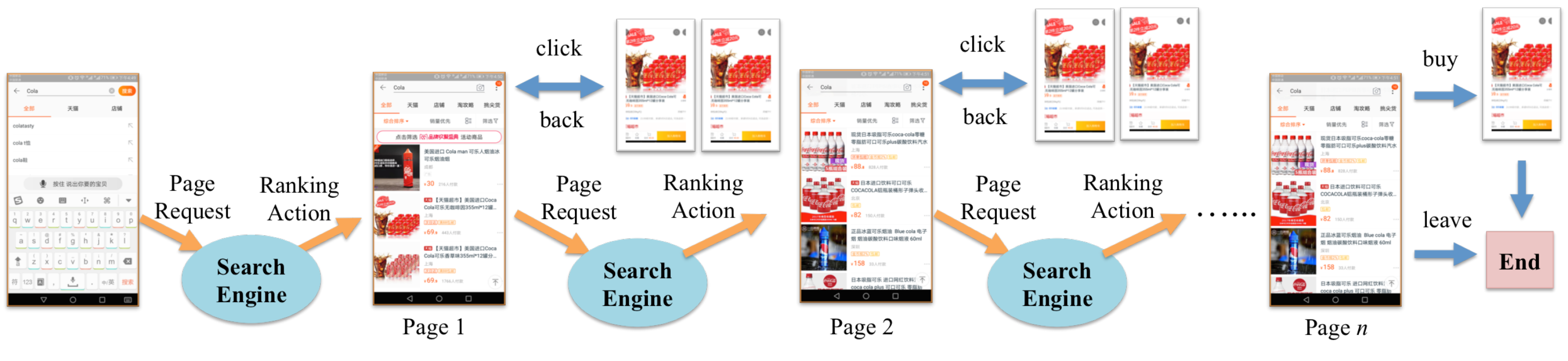}
  \vspace{-2mm}
  \caption{A typical search session in TaoBao. A user starts a session from a query, and has multiple actions to choose, including clicking into an item description, buying an item, turning to the next page, and leaving the session. \label{Fig:TaoBao_Search}}
\end{figure*}

\subsection{Related Work} \label{SubSec:2.2}
Early attempt of online LTR can be dated back to the evaluation of RankSVM in online settings \cite{joachims2002optimizing}. As claimed by \citeauthor{hofmann2013balancing}, balancing exploitation and exploration should be a key ability of online LTR methods \cite{hofmann2013balancing}. The theoretical results in the online learning community (typically in the bandit problem domain) \cite{auer2002using,langford2008epoch} provide rich mathematical tools for online LTR problem formalization and algorithms for efficient exploration, which motivates a lot of online LTR methods. In general, these methods can be divided into two groups. The first is to learn the best ranking function from a function space \cite{yue2009interactively,hofmann2013balancing}. For example, \citeauthor{yue2009interactively} \cite{yue2009interactively} define a dueling bandit problem in which actions are pairwise comparisons between documents and the goal is to learn a parameterized retrieval function which has sublinear regret performance. The second groups of online LTR methods directly learn the best list under some model of user interactions \cite{radlinski2008learning,slivkins2013ranked}, which can be treated as an assumption on how users act to a ranked list. Representative models include the cascade model \cite{kveton2015cascading,kveton2015combinatorial,zong2016cascading,li2016contextual}, the dependent-click model \cite{katariya2017stochastic}, and the position-based model \cite{lagree2016multiple}. Since no single model can entirely capture the behavior of all users, \citeauthor{zoghi2017online} \cite{zoghi2017online} recently propose a stochastic click learning framework for online LTR in a broad class of click models.

Our work in this paper is more similar to the first group of online LTR methods which learn ranking functions. However, while most of previous works consider a one-shot ranking problem, we focus on learning a ranking policy in a multi-step ranking problem, which contains multiple rounds of interactions and typically occurs in E-commerce scenarios.

\section{Problem Formulation} \label{Sec:3}
As we mentioned in previous sections, in E-commerce platforms such as \textit{TaoBao} and \textit{TMall}, ranking items given a query is a multi-step decision-making problem, where the search engine should take a ranking action whenever an item page is requested by a user. Figure (\ref{Fig:TaoBao_Search}) shows a typical search session between the search engine and a mobile app user in \textit{TaoBao}. In the beginning, the user inputs a query ``Cola'' into the blank of the search engine and clicks the ``Search'' button. Then the search engine takes a ranking action and shows the top items related to ``Cola'' in page 1. The user browses the displayed items and clicks some of them for the details. When no items interest the user or the user wants to check more items for comparisons, the user requests a new item page. The search engine again takes a ranking action and displays page 2. After a certain number of such ranking rounds, the search session will finally end when the user purchases items or just leaves the search session. 

\subsection{Search Session Modeling} \label{SubSec:3.1}
Before we formulate the multi-step ranking problem as an MDP, we define some concepts to formalize the contextual information and user behaviours in a search session, which are the basis for defining the state and state transitions of our MDP.

\begin{definition} [Top $K$ List] \label{Def:TopKList}
For an item set $\mathcal{D}$, a ranking function $f$, and a positive integer $K$ ($1 \leq K \leq |\mathcal{D}|$), the top $K$ list $\mathcal{L}_K(\mathcal{D}, f)$ is an ordered item list $(\mathcal{I}_1, \mathcal{I}_2, ..., \mathcal{I}_K)$ which contains the top $K$ items when applying the rank function $f$ to the item set $\mathcal{D}$, where $\mathcal{I}_k$ ($1 \leq k \leq K$) is the item in position $k$ and for any $k' \geq k$, it is the case that $f(\mathcal{I}_k) > f(\mathcal{I}_{k'})$.
\end{definition}

\begin{definition} [Item Page] \label{Def:ItemPage}
For each step $t$ ($t \geq 1$) during a session, the item page $p_{t}$ is the top $K$ list $\mathcal{L}_K(\mathcal{D}_{t-1}, a_{t-1})$ resulted by applying the ranking action $a_{t-1}$ of the search engine to the set of unranked items $\mathcal{D}_{t-1}$ in the last decision step $(t-1)$. For the initial step $t = 0$, $\mathcal{D}_0 = \mathcal{D}$. For any decision step $t \geq 1$, $\mathcal{D}_t = \mathcal{D}_{t-1} \setminus p_{t}$.
\end{definition}

\begin{definition} [Item Page History] \label{Def:ItemPageHistory}
In a search session, let $q$ be the input query. For the initial decision step $t = 0$, the initial item page history $h_0 = q$. For each later decision step $t \geq 1$, the item page history up to $t$ is $h_t = h_{t-1} \cup \{p_t\}$, where $h_{t-1}$ is the item page history up to the step $(t-1)$ and $p_t$ is the item page of step $t$.
\end{definition}

The item page history $h_t$ contains all information the user observes at the decision step $t$ $(t \geq 0)$. Since the item set $\mathcal{D}$ is finite, there are at most $\lceil \frac{|\mathcal{D}|}{K} \rceil$ item pages, and correspondingly at most $\lceil \frac{|\mathcal{D}|}{K} \rceil$ decision steps in a search session. In \textit{TaoBao} and \textit{TMall},  users may choose to purchase items or just leave at different steps of a session. If we treat all possible users as an environment which samples user behaviors, this would mean that after observing any item page history, the environment may terminate a search session with a certain probability of transaction conversion or abandonment. We formally define such two types of probability as follows.

\begin{definition} [Conversion Probability] \label{Def:ConversionProb}
For any item page history $h_t$ ($t > 0$) in a search session, let $B(h_t)$ denote the conversion event that a user purchases an item after observing $h_t$. The conversion probability of $h_t$, which is denoted by $b(h_t)$, is the averaged probability that $B(h_t)$ occurs when $h_t$ takes place.
\end{definition}

\begin{definition} [Abandon Probability] \label{Def:AbandonProb}
For any item page history $h_t$ ($t > 0$) in a search session, let $L(h_t)$ denote the abandon event that a user leaves the search session after observing $h_t$. The abandon probability of $h_t$, which is denoted by $l(h_t)$, is the averaged probability that $L(h_t)$ occurs when $h_t$ takes place.
\end{definition}

Since $h_t$ is the direct result of the agent's action $a_{t-1}$ in the last item page history $h_{t-1}$, the conversion probability $b(h_t)$ and the abandon probability $l(h_t)$ define how the state of the environment (i.e., the user population) will change after $a_{t-1}$ is taken in $h_{t-1}$: (1) terminating the search session by purchasing an item in $h_t$ with probability $b(h_t)$; (2) leaving the search session from $h_t$ with probability $l(h_t)$; (3) continuing the search session from $h_t$ with probability $(1 - b(h_t) - l(h_t))$. For convenience, we also define the continuing probability of an item page history.

\begin{definition} [Continuing Probability] \label{Def:ContProb}
For any item page history $h_t$ ($t \geq 0$) in a search session, let $C(h_t)$ denote the continuation event that a user continues searching after observing $h_t$. The continuing probability of $h_t$, which is denoted by $c(h_t)$, is the averaged probability that $C(h_t)$ occurs when $h_t$ takes place.
\end{definition}

Obviously, for any item page history $h$, it holds that $c(h) = 1 - b(h) - l(h)$. Specially, the continuation event of the initial item page history $h_0$ which only contains the query $q$ is a sure event (i.e., $c(h_0) = 1$) as neither a conversion event nor a abandon event can occur before the first item page is displayed.  

\subsection{Search Session MDP}
Now we are ready to define the instantiated Markov decision process (MDP) for the multi-step ranking problem in a search session, which we call a search session MDP (SSMDP). 
\begin{definition} [Search Session MDP] \label{Def:SSMDP}
Let $q$ be a query, $\mathcal{D}$ be the set of items related to $q$, and $K$ ($K > 0$) be the number of items that can be displayed in a page, the search session MDP (SSMDP) with respect to $q$, $\mathcal{D}$ and $K$ is a tuple $\mathcal{M} = \langle T, \mathcal{H}, \mathcal{S}, \mathcal{A}, \mathcal{R}, \mathcal{P} \rangle$, where 
\begin{itemize}         
\item[*] $T = \lceil \frac{|\mathcal{D}|}{K} \rceil$ is the maximal decision step of a search session,
\item[*] $\mathcal{H} = \bigcup_{t=0}^T \mathcal{H}_t$ is the set of all possible item page histories, $\mathcal{H}_t$ is the set of all item page histories up to $t$ ($0 \leq t \leq T$). 
\item[*] $\mathcal{S} = \mathcal{H}_C \bigcup \mathcal{H}_B \bigcup \mathcal{H}_L$ is the state space, $\mathcal{H}_C = \{ C(h_t) | \forall h_t \in \mathcal{H}_t, 0 \leq t < T \}$ is the nonterminal state set that contains all continuation events, $\mathcal{H}_B = \{ B(h_t) | \forall h_t \in \mathcal{H}_t, 0 < t \leq T \}$ and $\mathcal{H}_L = \{ L(h_t) | \forall h_t \in \mathcal{H}_t, 0 < t \leq T \}$ are two terminal state sets which contain all conversion events and all abandon events, respectively.
\item[*] $\mathcal{A}$ is the action space which contains all possible ranking functions of the search engine.
\item[*] $\mathcal{R} : \mathcal{H}_C \times \mathcal{A} \times \mathcal{S} \rightarrow \mathbb{R}$ is the reward function.
\item[*] $\mathcal{P} : \mathcal{H}_C \times \mathcal{A} \times \mathcal{S} \rightarrow [0,1]$ is the state transition function. For any step $t$ ($0 \leq t < T$), any item page history $h_t \in \mathcal{H}_t$, any action $a \in \mathcal{A}$, let $h_{t+1} = (h_t, \mathcal{L}_K(\mathcal{D}_t, a))$. The transition probability from the nonterminal state $C(h_t)$ to any state $s' \in \mathcal{S}$ after taking action $a$ is
\begin{equation} \label{Eq:StateTransition}
\mathcal{P}(C(h_t), a, s') = \left\{
\begin{aligned}
&b(h_{t+1}) &\text{if } s' = B(h_{t+1}), \\
&l(h_{t+1}) &\text{if } s' = L(h_{t+1}), \\
&c(h_{t+1}) &\text{if } s' = C(h_{t+1}), \\
&0 &\text{otherwise.} \quad\quad
\end{aligned}
\right.
\end{equation}
\end{itemize}
\end{definition}

In an SSMDP, the agent is the search engine and the environment is the population of all possible users. The states of the environment are indication of user status in the corresponding item page histories (i.e., contiuation, abandonment, or transaction conversion). The action space $\mathcal{A}$ can be set differently (e.g., discrete or continuous) according to  specific ranking tasks. The state transition function $\mathcal{P}$ is directly based on the conversion probability and abandon probability.
The reward function $\mathcal{R}$ highly depends on the goal of a specific task, we will discuss our reward setting in Section \ref{Sec:4.2}.

\section{Analysis of SSMDP} \label{Sec:4}
Before we apply the search session MDP (SSMDP) model in practice, some details need to be further clarified. In this section, we first identify the Markov property of the states in an SSMDP to show that SSMDP is well defined. Then we provide a reward function setting for SSMDP, based on which we perform an analysis on the reward discount rate and show the necessity for a search engine agent to maximize long-time accumulative rewards.

\subsection{Markov Property} \label{Sec:4.1}
The Markov property means that a state is able to summarize past sensations compactly in such a way that all relevant information is retained \cite{sutton1998reinforcement}. Formally, the Markov property refers to that for any state-action sequence $s_0, a_0, s_1, a_1, s_2, ..., s_{t-1}, a_{t-1}, s_t$ experienced in an MDP, it holds that 
\begin{equation}
\text{Pr}(s_t | s_0, a_0, s_1, a_1, ..., s_{t-1}, a_{t-1}) = \text{Pr}(s_t | s_{t-1}, a_{t-1}).
\end{equation}
That is to say, the occurring of the current state $s_t$ is only conditional on the last state-action pair $(s_{t-1}, a_{t-1})$ rather than the whole sequence. Now we show that the states of a search session MDP (SSMDP) also have the Markov property.

\begin{proposition}
For the search session MDP $\mathcal{M} = \langle T, \mathcal{H}, \mathcal{S}, \\ \mathcal{A}, \mathcal{R}, \mathcal{P} \rangle$ defined in Definition \ref{Def:SSMDP}, any state $s \in \mathcal{S}$ is Markovian.
\end{proposition}

\begin{proof}
We only need to prove that for any step $t$ ($0 \leq t \leq T$) and any possible state-action sequence $s_0, a_0, s_1, a_1, ..., s_{t-1}, a_{t-1}, s_t$ with respect to $t$, it holds that 
\begin{displaymath}
\text{Pr}(s_t | s_0, a_0, s_1, a_1, ..., s_{t-1}, a_{t-1}) = \text{Pr}(s_t | s_{t-1}, a_{t-1}).
\end{displaymath}
Note that all states except $s_t$ in the sequence $s_0, a_0, s_1, a_1, ..., s_{t-1}, \\ a_{t-1}, s_t$ must be non-terminal states. According to the state definition, for any step $t'$ ($0 < t' < t$), there must be an item page history $h_{t'}$ corresponding to the state $s_{t'}$ such that $s_{t'} = C(h(t'))$. So the state-action sequence can be rewritten as $C(h_0), a_0, C(h_1), a_1, ..., \\ C(h_{t-1}), a_{t-1}, s_t$. For any step $t'$ ($0 < t' < t$), it holds that 
\begin{displaymath}
h_{t'} = (h_{t' - 1}, \mathcal{L}_K(\mathcal{D}_{t'-1}, a_{t'-1})),
\end{displaymath}
where $\mathcal{L}_K(\mathcal{D}_{t'-1}, a_{t'-1})$ is the top $K$ list (i.e., item page) with respect to the unranked item set $\mathcal{D}_{t'-1}$ and ranking action $a_{t'-1}$ in step $(t'-1)$. Given $h_{t'-1}$, the unranked item set $\mathcal{D}_{t'-1}$ is deterministic. Thus, $h_{t'}$ is the necessary and unique result of the state-action pair $(C(h_{t'-1}), a_{t'-1})$. Therefore, the event $(C(h_{t'-1}), a_{t'-1})$ can be equivalently represented by the event $h_{t'}$, and the following derivation can be conducted: 
\begin{displaymath}
\begin{split}
&\text{Pr}(s_t | s_0, a_0, s_1, a_1, ..., s_{t-1}, a_{t-1}) \\
= &\text{Pr}(s_t | C(h_0), a_0, C(h_1), a_1, ..., C(h_{t-1}), a_{t-1}) \\
= &\text{Pr}(s_t | h_1, h_2, ..., h_{t-1}, C(h_{t-1}), a_{t-1}) \\
= &\text{Pr}(s_t | h_{t-1}, C(h_{t-1}), a_{t-1}) \\
= &\text{Pr}(s_t | C(h_{t-1}), a_{t-1}) \\
= &\text{Pr}(s_t | s_{t-1}, a_{t-1}).
\end{split}
\end{displaymath}
The third step of the derivation holds because for any step $t'$ ($0 < t' < t$), $h_{t'-1}$ is contained in $h_{t'}$. Similarly, the fourth step holds because $C(h_{t-1})$ contains the occurrence of $h_{t-1}$.
\end{proof}

\subsection{Reward Function} \label{Sec:4.2}
In a search session MDP $\mathcal{M} = \langle T, \mathcal{H}, \mathcal{S}, \mathcal{A}, \mathcal{R}, \mathcal{P} \rangle$, the reward function $\mathcal{R}$ of is a quantitative evaluation of the action performance in each state. Specifically, for any nonterminal state $s \in \mathcal{H}_C$, any action $a \in \mathcal{A}$, and any other state $s' \in \mathcal{S}$, $\mathcal{R}(s, a, s')$ is the expected value of the immediate rewards that numerically characterize the user feedback when action $a$ is taken in $s$ and the state is changed to $s'$. Therefore, we need to translate user feedback to numeric reward values that a learning algorithm can understand.

In the online LTR domain, user clicks are commonly adopted as a reward metric \cite{katariya2017stochastic,lagree2016multiple,zoghi2017online} to guide learning algorithms. However, in E-commerce scenarios, successful transactions between users (who search items) and sellers (whose items are ranked by the search engine) are more important than user clicks. Thus, our reward setting is designed to encourage more successful transactions. For any decision step $t$ ($0 \leq t < T$), any item page history $h_t \in \mathcal{H}_t$, and any action $a \in \mathcal{A}$, let $h_{t+1} = (h_t, \mathcal{L}_K(D_t, a))$. Recall that after observing the item page history $h_{t+1}$, a user will purchase an item with a conversion probability $b(h_{t+1})$. Although different users may choose different items to buy, from a statistical view, the deal prices of the transactions occurring in $h_{t+1}$ must follow an underlying distribution. We use $m(h_{t+1})$ to denote the expected deal price of $h_{t+1}$. Then for the nonterminal state $C(h_t)$ and any state $s' \in \mathcal{S}$, the reward $\mathcal{R}(C(h_t), a, s')$ is set as follows: 
\begin{equation} \label{Eq:RewardFunction}
\mathcal{R}(C(h_t), a, s') = \left\{
\begin{aligned}
&m(h_{t+1}) &\text{ if } s' = B(h_{t+1}), \\
&0 &\text{otherwise,} \quad\quad
\end{aligned}
\right.
\end{equation}
where $B(h_{t+1})$ is the terminal state which represents the conversion event of $h_{t+1}$. The agent will recieve a positive reward from the environment only when its ranking action leads to a successful transation. In all other cases, the reward is zero. It should be noted that the expected deal price of any item page history is most probably unknown beforehand. In practice, the actual deal price of a transaction can be directly used as the reward signal. 

\subsection{Discount Rate} \label{Sec:4.3}
The discount rate $\gamma$ is an important parameter of an MDP which defines the importance of future rewards in the objective of the agent (defined in Section \ref{SubSec:2.1}). For the search session MDP (SSMDP) defined in this paper, the choice of the discount rate $\gamma$ brings out a fundamental question: ``Is it necessary for the search engine agent to consider future rewards when making decisions?'' We will find out the answer and determine an appropriate value of the discount rate by analyzing how the objective of maximizing long-time accumulative rewards is related to the goal of improving the search engine's economic performance.

Let $\mathcal{M} = \langle T, \mathcal{H}, \mathcal{S}, \mathcal{A}, \mathcal{R}, \mathcal{P} \rangle$ be a search session MDP with respect to a query $q$, an item set $\mathcal{D}$ and an integer $K$ ($K > 0$). Given a fixed deterministic policy $\pi : \mathcal{S} \rightarrow \mathcal{A}$ of the agent\footnote{More accurately, the polic $\pi$ is a mapping from the nonterminal state set $\mathcal{H}_C$ to the action space $\mathcal{A}$. Our conclusion in this paper also holds for stochastic policies, but we ommit the discussion due to space limitation.}, denote the item page history occurring at step $t$ ($0 \leq t \leq T$) under $\pi$ by $h^{\pi}_t$. We enumerate all possible states that can be visited in a search session under $\pi$ in Figure \ref{Fig:Markov_Chain}. For better illustration, we show all item page histories (marked in red) in the figure. Note that they are not the states of the SSMDP $\mathcal{M}$. Next, we will rewrite $C(h^{\pi}_t)$, $c(h^{\pi}_t)$, $b(h^{\pi}_t)$, and $m(h^{\pi}_t)$ as $C^{\pi}_t$, $c^{\pi}_t$, $b^{\pi}_t$, and $m^{\pi}_t$ for simplicity.

\begin{figure}
  \centering
  \includegraphics[scale=0.26]{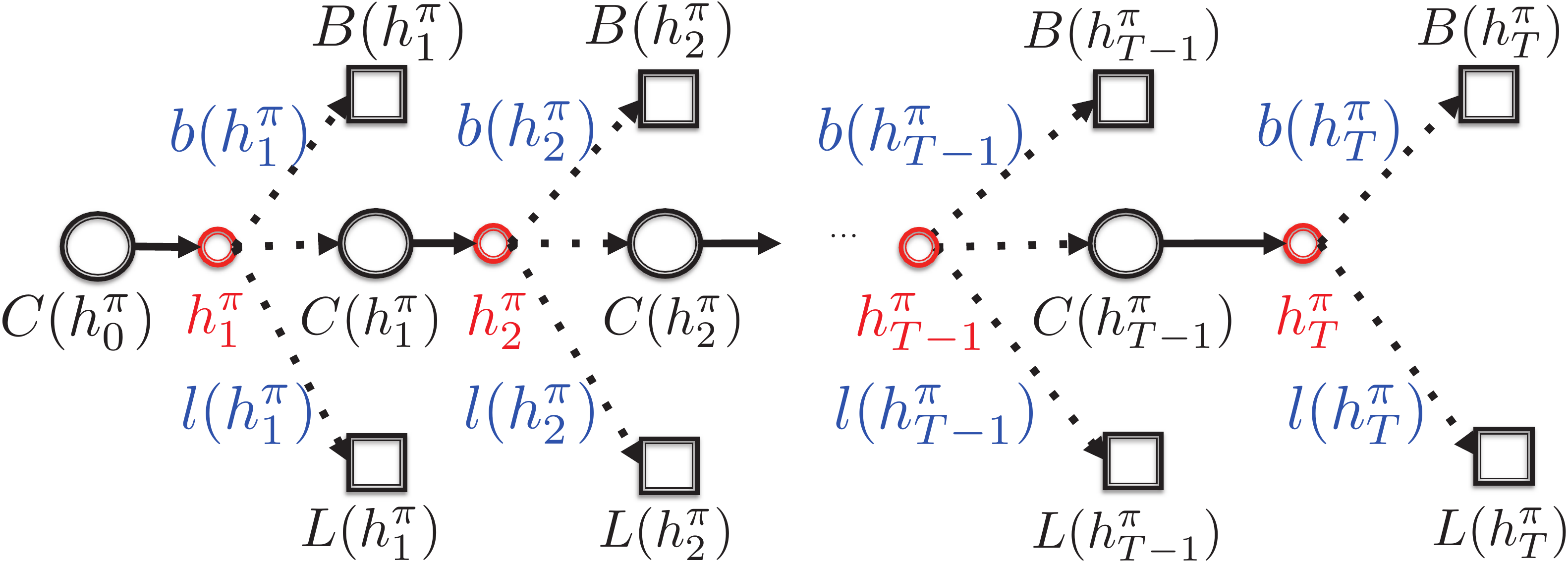}
  \vspace{-4mm}
  \caption{All states that can be visited under policy $\pi$. The black circles are nonterminal states and the black squares are terminal states. The red circles are item page histories. The solid black arrow starting from each nonterminal state represents the execution of the policy $\pi$. The dotted arrows from each item page history are state transitions, with the corresponding transition probabilities marked in blue. \label{Fig:Markov_Chain}}
  \vspace{-5mm}
\end{figure}

Without loss of generality, we assume the discount rate of the SSMDP $\mathcal{M}$ is $\gamma$ ($0 \leq \gamma \leq 1$). Denote the state value function (i.e., expected accumulative rewards) under $\gamma$ by $V^{\pi}_{\gamma}$ . For each step $t$ ($0 \leq t < T$), the state value of the nonterminal state $C^{\pi}_t$ is 
\begin{equation} \label{Eq:ValueFunction}
\begin{split}
V^{\pi}_{\gamma}(C^{\pi}_t) &= \mathbb{E}^{\pi} \big\{ \sum_{k=1}^{T-t} \gamma^{k-1} r_{t+k} \big| C^{\pi}_t \big\} \\
&= \mathbb{E}^{\pi} \big\{ r_{t+1} + \gamma r_{t+2} + \cdot\cdot\cdot + \gamma^{T-t-1} r_{T} \big| C^{\pi}_t \big\},
\end{split}
\end{equation}
where for any $k$ ($1 \leq k \leq T-t$), $r_{t+k}$ is the immediate reward recieved at the future step $(t+k)$ in the item page history $h^{\pi}_{t+k}$. According to the reward function in Equation (\ref{Eq:RewardFunction}), the expected value of the immediate reward $r_{t+k}$ under $\pi$ is 
\begin{equation} \label{Eq:ExpImmRwd}
\mathbb{E}^{\pi} \big\{ r_{t+k} \big\} = b^{\pi}_{t+k} m^{\pi}_{t+k},
\end{equation}
where $m^{\pi}_{t+k} = m(h^{\pi}_{t+k})$ is the expected deal price of the item page history $h^{\pi}_{t+k}$. However, since $V^{\pi}_{\gamma}(C^{\pi}_t)$ is the expected discounted accumulative rewards on condition of the state $C^{\pi}_t$, the probability that the item page history $h^{\pi}_{t+k}$ is reached when $C^{\pi}_t$ is visited should be taken into account. Denote the reaching probability from $C^{\pi}_t$ to $h^{\pi}_{t+k}$ by $\text{Pr}( C^{\pi}_t \rightarrow h^{\pi}_{t+k})$, it can be computed as follows according to the state transition function in Equation (\ref{Eq:StateTransition}): 
\begin{equation} \label{Eq:ReachingProb}
\text{Pr}(C^{\pi}_t \rightarrow h^{\pi}_{t+k}) = \left\{
\begin{aligned}
&1.0 &k = 1, \quad\quad\quad \\
&\Pi^{k-1}_{j=1} c^{\pi}_{t+j} &1 < k \leq T-t.
\end{aligned}
\right.
\end{equation}
The reaching probability from $C^{\pi}_t$ to $h^{\pi}_{t+1}$ is $1$ since $h^{\pi}_{t+1}$ is the directly result of the state action pair $(C^{\pi}_t, \pi(C^{\pi}_t))$. For other future item page histories, the reaching probability is the product of all continuing probabilities along the path from $C^{\pi}_{t+1}$ to $C^{\pi}_{t+k-1}$. By taking Equations (\ref{Eq:ExpImmRwd}) and (\ref{Eq:ReachingProb}) into Equation (\ref{Eq:ValueFunction}), $V^{\pi}_{\gamma}(C^{\pi}_t)$ can be further computed as follows:
\begin{equation} \label{Eq:ValueFunctionCont}
\begin{split}
V^{\pi}_{\gamma}(C^{\pi}_t) &= \mathbb{E}^{\pi} \big\{ r_{t+1} \big| C^{\pi}_t \big\} + \gamma \mathbb{E}^{\pi} \big\{r_{t+2} \big| C^{\pi}_t \big\} + \cdot\cdot\cdot \\
&+ \gamma^{k-1} \mathbb{E}^{\pi} \big\{ r_{t+k} \big| C^{\pi}_t \big\} + \cdot\cdot\cdot + \gamma^{T-t-1} \mathbb{E}^{\pi} \big\{ r_{T} \big| C^{\pi}_t \big\} \\
&= \sum^{T-t}_{k=1} \gamma^{k-1} \text{Pr}(C^{\pi}_t \rightarrow h^{\pi}_{t+k}) b^{\pi}_{t+k} m^{\pi}_{t+k} \\
&= b^{\pi}_{t+1} m^{\pi}_{t+1} + \sum^{T-t}_{k=2} \gamma^{k-1} \Big( \big( \Pi^{k-1}_{j=1} c^{\pi}_{t+j} \big) b^{\pi}_{t+k} m^{\pi}_{t+k} \Big).
\end{split}
\end{equation}
With the conversion probability and the expected deal price of each item page history in Figure \ref{Fig:Markov_Chain}, we can also derive the expected gross merchandise volume (GMV) lead by the search engine agent in a search session under the policy $\pi$ as follows: 
\begin{equation} \label{Eq:GMV}
\begin{split}
\mathbb{E}^{\pi}_{\text{gmv}} &= b^{\pi}_{1} m^{\pi}_1 + c^{\pi}_1 b^{\pi}_2 m^{\pi}_2 + \cdot\cdot\cdot + \big( \Pi_{k=1}^{T} c^{\pi}_k \big) b^{\pi}_T m^{\pi}_T \\
&= b^{\pi}_{1} m^{\pi}_1 + \sum_{k=2}^{T} \big( \Pi_{j=1}^{k-1} c^{\pi}_j \big) b^{\pi}_k m^{\pi}_k.
\end{split}
\end{equation}
By comparing Equations (\ref{Eq:ValueFunctionCont}) and (\ref{Eq:GMV}), it can be easily found that $\mathbb{E}^{\pi}_{gmv} = V^{\pi}_{\gamma}(C^{\pi}_0)$ when the discount rate $\gamma = 1$. That is to say, when $\gamma = 1$, maximizing the expected accumulative rewards directly leads to the maximization of the expected GMV. However, when $\gamma < 1$, maximizing the value function $V^{\pi}_{\gamma}$ cannot necessarily maximize $\mathbb{E}^{\pi}_{gmv}$ since the latter is an upper bound of $V^{\pi}_{\gamma}(C^{\pi}_0)$.

\begin{proposition}
Let $\mathcal{M} = \langle T, \mathcal{H}, \mathcal{S}, \mathcal{A}, \mathcal{R}, \mathcal{P} \rangle$ be a search session MDP. For any deterministic policy $\pi : \mathcal{S} \rightarrow \mathcal{A}$ and any discount rate $\gamma$ ($0 \leq \gamma \leq 1$), it is the case that $V^{\pi}_{\gamma}(C(h_0)) \leq \mathbb{E}^{\pi}_{gmv}$, where $V^{\pi}_{\gamma}$ is state value function defined in Equation (\ref{Eq:ValueFunction}), $C(h_0)$ is the initial nonterminal state of a search session, $\mathbb{E}^{\pi}_{gmv}$ is the expected gross merchandise volume (GMV) of $\pi$ defined in Equation (\ref{Eq:GMV}). Only when $\gamma = 1$, we have $V^{\pi}_{\gamma}(C(h_0)) = \mathbb{E}^{\pi}_{gmv}$.
\end{proposition}
\begin{proof}
The proof is trivial since the difference between $\mathbb{E}^{\pi}_{gmv}$ and $V^{\pi}_{\gamma}(C(h_0))$, namely $\sum_{k=2}^{T} (1 - \gamma^{k-1} ) \big( \Pi_{j=1}^{k-1} c^{\pi}_j \big) b^{\pi}_k m^{\pi}_k$, is always positive when $\gamma < 1$.
\end{proof}
Now we can give the answer to the question proposed in the beginning of this section: considering future rewards in a search session MDP is necessary since maximizing the undiscounted expected accumulative rewards can optimize the performance of the search engine in the aspect of GMV. The sequential nature of our multi-step ranking problem requires the ranking decisions at different steps to be optimized integratedly rather than independently.

\section{Algorithm} \label{Sec:5}
In this section, we propose a policy gradient algorithm for learning an optimal ranking policy in a search session MDP (SSMDP). 
We resort to the policy gradient method since directly optimizing a parameterized policy function addresses both the policy representation issue and the large-scale action space issue of an SSMDP. Now we briefly review the policy gradient method in the context of SSMDP. Let $\mathcal{M} = \langle T, \mathcal{H}, \mathcal{S}, \mathcal{A}, \mathcal{R}, \mathcal{P} \rangle$ be an SSMDP, $\pi_{\theta}$ be the policy function with the parameter $\theta$. The objective of the agent is to find an optimal parameter which maximizes the expectation of the $T$-step returns along all possible trajectories
\begin{equation}
J(\theta) = \mathbb{E}_{\tau \sim \rho_{\theta}} \big\{ R(\tau) \big\} = \mathbb{E}_{\tau \sim \rho_{\theta}} \big\{ \sum_{t=0}^{T-1} r_t \big\},
\end{equation}
where $\tau$ is a trajectory like $s_0, a_0, r_0, s_1, a_1, ..., s_{T-1}, a_{T-1}, r_{T-1}, s_T$ and follows the trajectory distribution $\rho_{\theta}$ under the policy parameter $\theta$, $R(\tau) = \sum_{t=0}^{T-1} r_t$ is the $T$-step return of the trajectory $\tau$. Note that if the terminal state of a trajectory is reached in less than $T$ steps, the sum of the rewards will be truncated in that state. The gradient of the target $J(\theta)$ with respect to $\theta$ is 
\begin{equation}
\nabla_{\theta} J(\theta) = \mathbb{E}_{\tau \sim \rho_{\theta}} \big\{ \sum_{t=0}^{T-1} \nabla_{\theta} \log \pi_{\theta}(s_{t}, a_{t}) R_t^T(\tau) \big\},
\end{equation}
where $R_t^T(\tau) = \sum_{t'=t}^{T-1} r_{t'}$ is the sum of rewards from step $t$ to the terminal step $T$ in the trajectory $\tau$. This gradient leads to the well-known REINFORCE algorithm \cite{williams1992simple}. The policy gradient theorem proposed by \citeauthor{sutton2000policy} \cite{sutton2000policy} provides a framework which generalizes the REINFORCE algorithm. In general, the gradient of $J(\theta)$ can be written as 
\begin{displaymath}
\nabla_{\theta} J(\theta) = \mathbb{E}_{\tau \sim \rho_{\theta}} \big\{ \sum_{t=0}^{T} \nabla_{\theta} \log \pi_{\theta}(s_{t},a_{t}) Q^{\pi_{\theta}}(s_{t},a_{t}) \big\},
\end{displaymath}
where $Q^{\pi_{\theta}}$ is the state-action value function under the policy $\pi_{\theta}$. If $\pi_{\theta}$ is deterministic, the gradient of $J(\theta)$ can be rewritten as  
\begin{displaymath}
\nabla_{\theta} J(\theta) = \mathbb{E}_{\tau \sim \rho_{\theta}} \big\{ \sum_{t=0}^{T-1} \nabla_{\theta} \pi_{\theta}(s_{t}) \nabla_{a} Q^{\pi_{\theta}}(s_{t},a) {\big|}_{a=\pi_{\theta}(s_{t})} \big\}.
\end{displaymath}
\citeauthor{silver2014deterministic} \cite{silver2014deterministic} show that the deterministic policy gradient is the limiting case of the stochastic policy gradient as policy variance tends to zero. The value function $Q^{\pi_{\theta}}$ can be estimated by temporal-difference learning (e.g., actor-critic methods \cite{sutton1998reinforcement}) aided by a function approximator $Q^{w}$ with the parameter $w$ which minimizes the mean squared error $\text{MSE}(w) = \vert \vert Q^{w} - Q^{\pi_{\theta}} \vert \vert^2$.

\subsection{The DPG-FBE Algorithm}
Instead of using stochastic policy gradient algorithms, we rely on the deterministic policy gradient (DPG) algorithm \cite{silver2014deterministic} to learn an optimal ranking policy in an SSMDP since from a practical viewpoint, computing the stochastic policy gradient may require more samples, especially if the action space has many dimensions. However, we have to overcome the difficulty in estimating the value function $Q^{\pi_{\theta}}$, which is caused by the high variance and unbalanced distribution of the immediate rewards in each state. As indicated by Equation (\ref{Eq:RewardFunction}), the immediate reward of any state-action pair $(s,a)$ is zero or the expected deal price $m(h)$ of the item history page $h$ resulted by $(s,a)$. Firstly, the reward variance is high because the deal price $m(h)$ normally varies over a wide range. Secondly, the immediate reward distribution of $(s,a)$ is unbalanced because the conversion events lead by $(s,a)$ occur much less frequently than the two other cases (i.e., abandon and continuation events) which produce zero rewards. Note that the same problem also exists for the $T$-step returns of the trajectories in an SSMDP since in any possible trajectory, only the reward of the last step may be nonzero. Therefore, estimating $Q^{\pi_{\theta}}$ by Monte Carlo evaluation or temporal-difference learning may cause inaccurate update of the value function parameters and further influence the optimization of the policy parameter.

Our way for solving the above problem is similar to the model-based reinforcement learning approaches \cite{kearns2002near,BrafmanT02}, which maintain an approximate model of the environment to help with performing reliable updates of value functions. According to the Bellman Equation \cite{sutton1998reinforcement}, the state-action value of any state-action pair $(s,a)$ under any policy $\pi$ is 
\begin{equation} \label{Eq:BellmanEqQ}
Q^{\pi_{\theta}}(s,a) = \sum_{s' \in \mathcal{S}} \mathcal{P}(s,a,s') \big( \mathcal{R}(s,a,s') + \max_{a'} Q^{\pi_{\theta}}(s',a') \big),
\end{equation}
The right-hand side of Equation (\ref{Eq:BellmanEqQ}) can be denoted by $\mathcal{T} Q^{\pi_{\theta}}(s,a)$, where $\mathcal{T}$ is the Bellman operator with respect to the policy $\pi_{\theta}$. Let $h'$ be the next item page history resulted by $(s,a)$. Only the states $C(h')$, $B(h')$, and $L(h')$ can be transferred to from $(s,a)$ with nonzero probability. Among these three states, only $B(h')$ involves a nonzero immediate reward and $C(h')$ involves a nonzero $Q$-value. So the above equation can be simplified to 
\begin{equation} \label{Eq:BellmanEqSimplified}
Q^{\pi_{\theta}}(s,a) = b(h') m(h') + c(h') \max_{a'} Q^{\pi_{\theta}}\left(C(h'),a'\right),
\vspace{-1mm}
\end{equation}
where $b(h')$, $c(h')$, and $m(h')$ are the conversion probability, continuing probability and expected deal price of $h'$, respectively. Normally, the value function $Q^{\pi_{\theta}}$ can be approximated by a parameterized function $Q^w$ with an objective of minimizing the mean squared error (MSE)
\begin{displaymath}
\text{MSE}(w) = \vert \vert Q^{w} - Q^{\pi_{\theta}} \vert \vert^2 = \sum_{s \in \mathcal{S}} \sum_{a \in \mathcal{A}} \big( Q^{w}(s,a) - Q^{\pi_{\theta}}(s,a) \big)^2 .
\end{displaymath}
The derivative of $\text{MSE}(w)$ with respect to the parameter $w$ is 
\begin{displaymath}
\nabla_w \text{MSE}(w) = \sum_{s \in \mathcal{S}} \sum_{a \in \mathcal{A}} \big( Q^{\pi_{\theta}}(s,a) - Q^{w}(s,a) \big) \nabla_w Q^{w}(s,a).
\end{displaymath}
However, since $Q^{\pi_{\theta}}(s,a)$ is unknown, we cannot get the accurate value of $\nabla_w \text{MSE}(w)$. One way for solving this problem is to replace $Q^{\pi_{\theta}}$ with $\mathcal{T} Q^w$ and approximately compute $\nabla_w \text{MSE}(w)$ by 
\begin{equation*}
\begin{small}
\begin{aligned}
\sum_{s \in \mathcal{S}} \sum_{a \in \mathcal{A}} \Big( b(h')m(h') + c(h') \max_{a'} Q^{w}(s',a') - Q^{w}(s,a) \Big) \nabla_w Q^{w}(s,a),
\end{aligned}
\end{small}
\end{equation*}
where $s' = C(h')$ is the state of continuation event of $h'$. Every time a state-action pair $(s,a)$ as well as its next item page history $h'$ is observed, $w$ can be updated in a full backup manner: 
\begin{displaymath}
\Delta w \leftarrow \alpha_w \nabla_{w} Q^w(s,a) \big( b(h') m(h') + c(h') Q^w(s',a') - Q^w(s,a) \big),
\end{displaymath}
where $\alpha_w$ is a learning rate and $a' = \pi_{\theta}(s')$. With this full backup updating method, the sampling errors caused by immediate rewards or returns can be avoided. Furthermore, the computational cost of full backups in our problem is almost equal to that of one-step sample backups (e.g., Q-learning \cite{watkins1989learning}). 

\begin{algorithm}[t]
\caption{Deterministic Policy Gradient with Full Backup Estimation (DPG-FBE)}
\label{Alg:DPGFBE}
\begin{small}
\KwIn{Learning rate $\alpha_{\theta}$ and $\alpha_{w}$, pretrained conversion probability model $b$, continuing probability model $c$, and expected deal price model $m$ of item page histories}
Initialize the actor $\pi_{\theta}$ and the critic $Q^w$ with parameter $\theta$ and $w$\;
\ForEach{search session}{
  Use $\pi_{\theta}$ to sample a ranking action at each step with exploration\;
  Get the trajectory $\tau$ of the session with its final step index $t$\;
  $\Delta w \leftarrow 0, \Delta \theta \leftarrow 0$\;
  \For{$k = 0, 1, 2, ..., t-1$}{
    $(s_k, a_k, r_{k}, s_{k+1}) \leftarrow $ the sample tuple at step $k$\;
    $h_{k+1} \leftarrow $ the item page history of $s_{k}$\;
    \eIf{$s_{k+1} = B(h_{k+1})$}{
      Update the models $b$, $c$, and $m$ with the samples $(h_{k+1},1)$, $(h_{k+1},0)$, and $(h_{k+1}, r_{k})$, respectively\;
    }{
      Update the models $b$ and $c$ with the samples $(h_{k+1},0)$ and $(h_{k+1},1)$, respectively\;
    }
    $s' \leftarrow C(h_{k+1})$, $a' \leftarrow \pi_{\theta}(s')$\;
    $p_{k+1} \leftarrow b(h_{k+1}) m(h_{k+1})$\;
    $\delta_k \leftarrow p_{k+1} + c(h_{k+1}) Q^w(s',a') - Q^w(s_{k},a_{k})$\;
    $\Delta w \leftarrow \Delta w + \alpha_w \delta_k \nabla_{w} Q^w(s_{k},a_{k})$\;
    $\Delta \theta \leftarrow \Delta \theta + \alpha_{\theta} \nabla_{\theta} \pi_{\theta}(s_k) \nabla_{a} Q^w(s_k, a_k)$\;
  }
  $w \leftarrow w + \Delta w / t, \theta \leftarrow \theta + \Delta \theta / t$\;
}
\end{small}
\end{algorithm}

Our policy gradient algorithm is based on the deterministic policy gradient theorem \cite{silver2014deterministic} and the full backup estimation of the Q-value functions. Unlike previous works which entirely model the reward and state transition functions \cite{kearns2002near,BrafmanT02}, we only need to build the conversion probability model $b(\cdot)$, the continuing probability model $c(\cdot)$, and the expected deal price model $m(\cdot)$ of the item page histories in an SSMDP. These models can be trained using online or offline data by any possible statistical learning method. We call our algorithm {Deterministic Policy Gradient with Full Backup Estimation} (DPG-FBE) and show its details in Algorithm 1.

As shown in this table, the parameters $\theta$ and $w$ will be updated after any search session between the search engine agent and users. Exploration (at line $3$) can be done by, but not limited to, $\epsilon$-greedy (in discrete action case) or adding random noise to the output of $\pi_{\theta}$ (in continuous action case). Although we have no assumptions on the specific models used for learning the actor $\pi_{\theta}$ and the critic $Q^w$ in Algorithm \ref{Alg:DPGFBE}, nonlinear models such as neural networks are preferred due to the large state/action space of an SSMDP. To solve the convergence problem and ensure a stable learning process, a replay buffer and target updates are also suggested \cite{mnih2015human,lillicrap2015continuous}.

\section{Experiments}\label{Sec:6}
In this section, we conduct two groups of experiments: a simulated experiment in which we construct an online shopping simulator and test our algorithm DPG-FBE as well as some state-of-the-art online learning to rank (LTR) algorithms, and a real application in which we apply our algorithm in \textit{TaoBao}, one of the largest E-commerce platforms in the world. 

\subsection{Simulation}
The online shopping simulator is constructed based on the statistical information of items and user behaviors in \textit{TaoBao}. An item is represented by a $n$-dim ($n > 0$) feature vector  $\bm{x} = (x_1,...,x_n)^{\top}$ and a ranking action of the search engine is a $n$-dim weight vector $\bm{\mu} = (\mu_1,...,\mu_n)^{\top}$. The ranking score of the item $\bm{x}$ under the ranking action $\bm{\mu}$ is the inner product $\bm{x}^{\top} \bm{\mu}$ of the two vectors. We choose $20$ important features related to the item category of \textit{dress} (e.g., price and quality) and generate an item set $\mathcal{D}$ by sampling $1000$ items from a distribution approximated with all the items of the dress category. Each page contains $10$ items so that there are at most $100$ ranking rounds in a search session. In each ranking round, the user operates on the current item page (such as clicks, abandonment, and purchase) are simulated by a user behavior model, which is constructed from the user behavior data of the dress items in \textit{TaoBao}. The simulator outputs the probability of each possible user operation given the recent item pages examined by the user. A search session will end when the user purchases one item or leaves.

\begin{figure}
  \centering
  \includegraphics[scale=0.55,trim=0 15 0 6,clip]{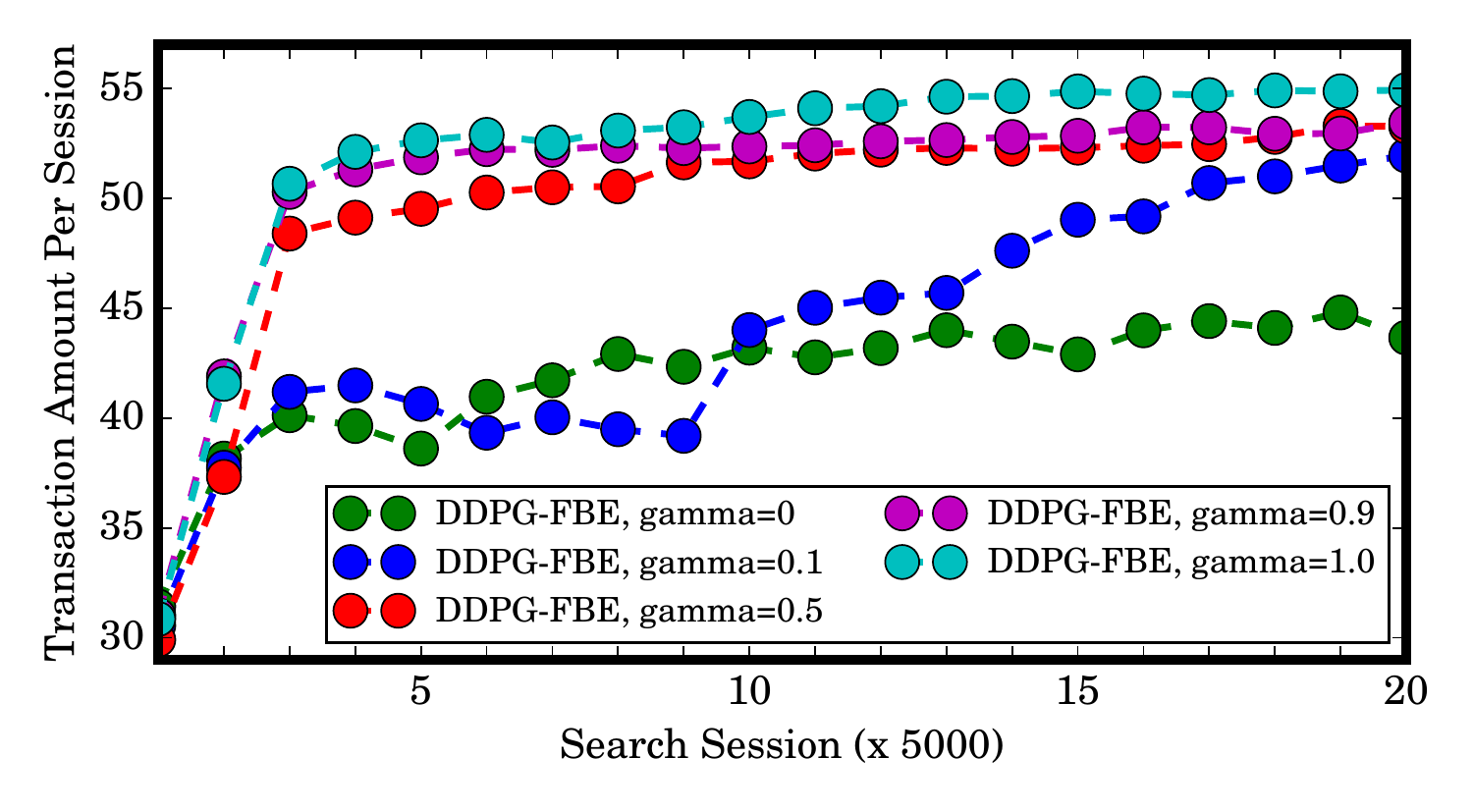}
  \vspace{-2mm}
  \caption{The learning performance of the DDPG-FBE algorithm in the simulation experiment \label{Fig:Sim_Results_DDPGFBE}}
  \vspace{-2mm}
\end{figure}

Our implementation of the DPG-FBE algorithm is a deep RL version (DDPG-FBE) which adopts deep neural networks (DNN) as the policy and value function approximators (i.e., actor and critic). We also implement the deep DPG algorithm (DDPG) \cite{lillicrap2015continuous}. The state of environment is represented by a $180$-dim feature vector extracted from the last $4$ item pages of the current search session. The actor and critic networks of the two algorithms have two full connected hidden layers with $200$ and $100$ units, respectively. We adopt \textit{relu} and \textit{tanh} as the activation functions for the hidden layers and the output layers of all networks. The network parameters are optimized by Adam with a learning rate of $10^{-5}$ for the actor and $10^{-4}$ for the critic. The parameter $\tau$ for the soft target updates \cite{lillicrap2015continuous} is set to $10^{-3}$. We test the performance of the two algorithms under different settings of the discount rate $\gamma$. Five online LTR algorithms, point-wise LTR, BatchRank \cite{zoghi2017online}, CascadeUCB1 \cite{kveton2015cascading}, CascadeKL-UCB \cite{kveton2015cascading}, and RankedExp3 \cite{radlinski2008learning} are implemented for comparison. Like the two RL algorithms, the point-wise LTR method implemented in our simulation also learns a parameterized function which outputs a ranking weight vector in each state of a search session. We choose DNN as the parameterized function and use the logistic regression algorithm to train the model, with an objective function that approximates the goal of maximizing GMV. The four other online LTR algorithms are regret minimization algorithms which are based on variants of the bandit problem model. The test of each algorithm contains $100,000$ search sessions and the transaction amount of each session is recorded. Results are averaged over $50$ runs and are shown in Figures \ref{Fig:Sim_Results_DDPGFBE}, \ref{Fig:Sim_Results_DDPG}, and \ref{Fig:Sim_Results_OLTR}.

\begin{figure}
  \centering
  \includegraphics[scale=0.55,trim=0 15 0 6,clip]{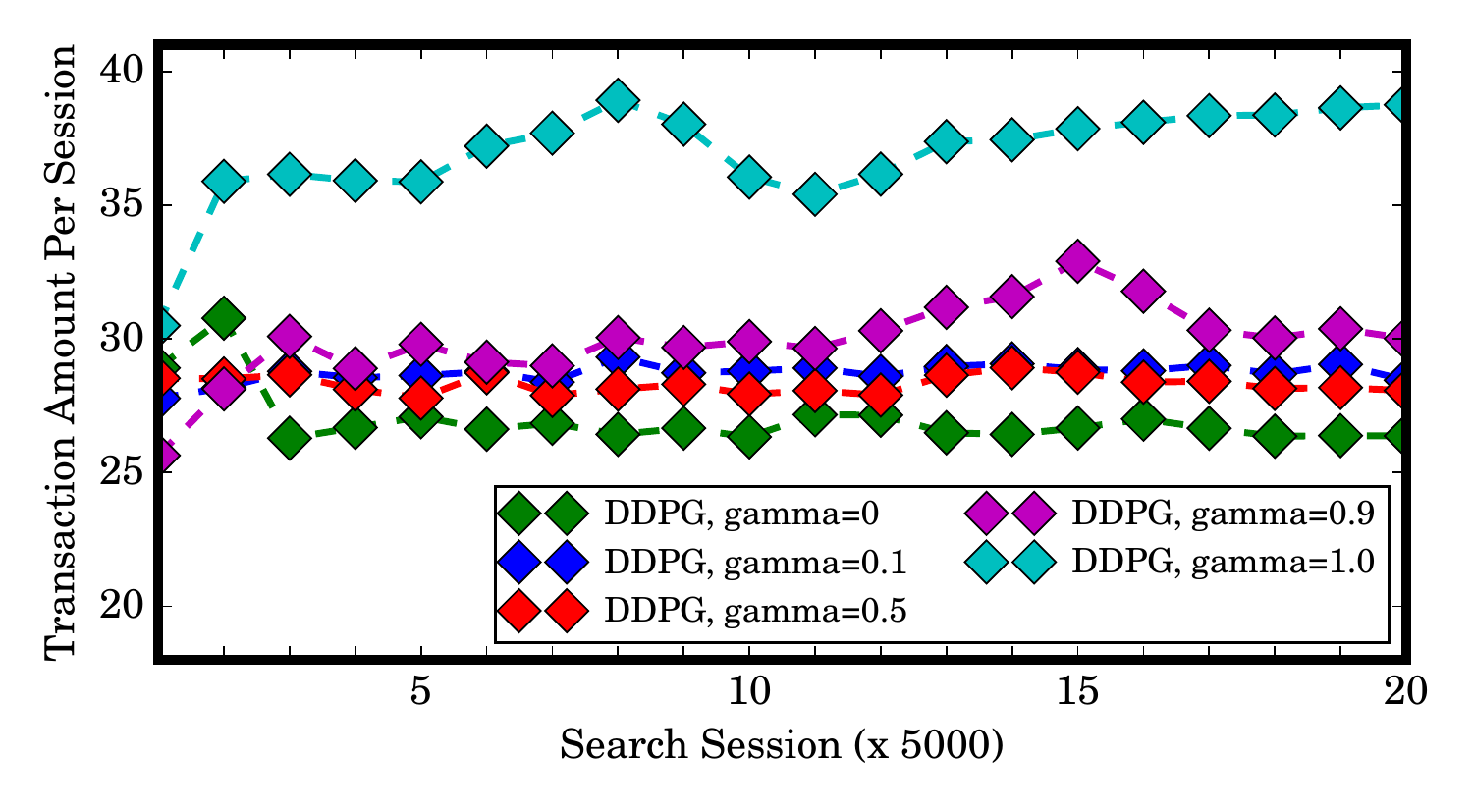}
  \vspace{-2mm}
  \caption{The learning performance of the DDPG algorithm in the simulation experiment \label{Fig:Sim_Results_DDPG}}
  \vspace{-2mm}
\end{figure}

Now let us first examine the figure of DDPG-FBE. It can be found that the performance of DDPG-FBE is improved as the discount rate $\gamma$ increases. The learning curve corresponding to the setting $\gamma = 0$ (the green one) is far below other curves in Fig. \ref{Fig:Sim_Results_DDPGFBE}, which indicates the importance of delay rewards. The theoretical result in Section \ref{Sec:4} is empirically verified since the DDPG-FBE algorithm achieves the best performance when $\gamma = 1$, with $2\%$ growth of transaction amount per session compared to the second best performance. Note that in E-commerce scenarios, even $1\%$ growth is considerable. The DDPG algorithm also performs the best when $\gamma = 1$, but it fails to learn as well as the DDPG-FBE algorithm. As shown in Fig. \ref{Fig:Sim_Results_DDPG}, all the learning curves of DDPG are under the value $40$. The point-wise LTR method also outputs a ranking weight vector while the other four online LTR algorithms can directly output a ranked item list according to their own ranking mechanisms. However, as we can observe in Fig. \ref{Fig:Sim_Results_OLTR}, the transaction amount lead by each of the algorithms is much smaller than that lead by DDPG-FBE and DDPG. This is not surprising since none of these algorithms are designed for the multi-step ranking problem where the ranking decisions at different steps should be optimized integratedly.

\subsection{Application}
We apply our algorithm in \textit{TaoBao} search engine for providing online realtime ranking service. The searching task in \textit{TaoBao} is characterized by high concurrency and large data volume. In each second, the \textit{TaoBao} search engine should respond to hundreds of thousands of users' requests in concurrent search sessions and simultaneously deal with the data produced from user behaviours. On sale promotion days such as the \textit{TMall Double $11$ Global Shopping Festival}\footnote{This refers to November the 11-th of each year. On that day, most sellers in \textit{TaoBao} and \textit{TMall} carry out sale promotion and billions of people in the world join in the online shopping festival}, both the volume and producing rate of the data would be multiple times larger than the daily values.

\begin{figure}
  \centering
  \includegraphics[scale=0.55,trim=0 15 0 6,clip]{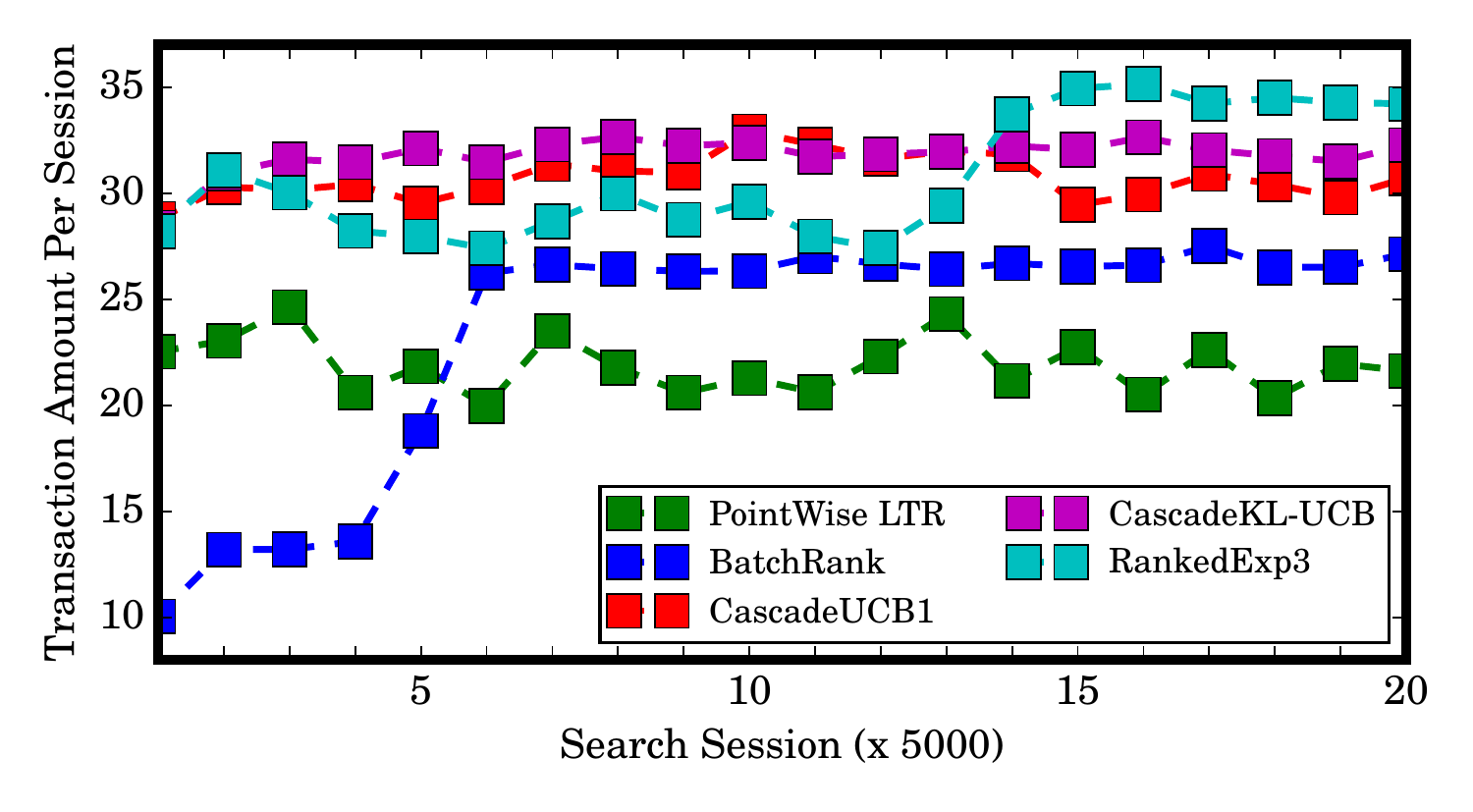}
  \vspace{-2mm}
  \caption{The learning performance of five online LTR algorithms in the simulation experiment \label{Fig:Sim_Results_OLTR}}
  \vspace{-2mm}
\end{figure}

\begin{figure}
  \centering
  \includegraphics[scale=0.28]{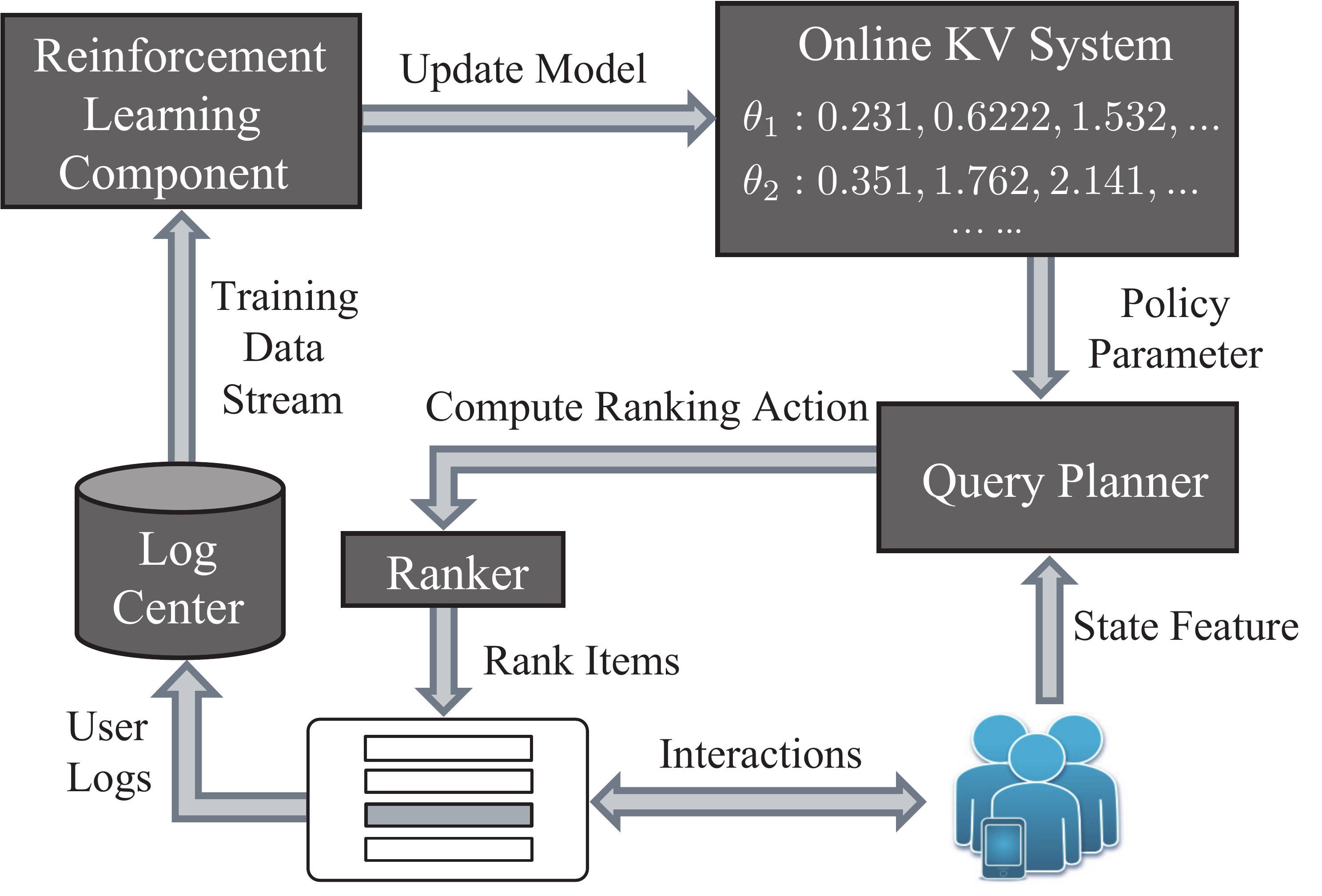}
  \vspace{-4mm}
  \caption{RL ranking system of \textit{TaoBao} search engine \label{Fig:RL_Ranking_System}}
  \vspace{-3mm}
\end{figure}

In order to satisfy the requirement of high concurrency and the ability of processing massive data in \textit{TaoBao}, we design a data stream-driven RL ranking system for implementing our algorithm DPG-FBE. As shown in Figure \ref{Fig:RL_Ranking_System}, this system contains five major components: a query planner, a ranker, a log center, a reinforcement learning component, and an online KV system. The work flow of our system mainly consists of two loops. The first one is an online acting loop (in the right bottom of Figure \ref{Fig:RL_Ranking_System}), in which the interactions between the search engine and \textit{TaoBao} users take place. The second one is a learning loop (on the left of the online acting loop in Figure \ref{Fig:RL_Ranking_System}) where the training process happens. The two working loops are connected through the log center and the online KV system, which are used for collecting user logs and storing the ranking policy model, respectively. In the first loop, every time a user requests an item page, the query planner will extract the state feature, get the parameters of the ranking policy model from the online KV system, and compute a ranking action for the current state (with exploration). The ranker will apply the computed action to the unranked items and display the top $K$ items (e.g., $K = 10$) in an item page, where the user will give feedback. In the meanwhile, the log data produced in the online acting loop is injected into the learning loop for constructing training data source. In the log center, the user logs collected from different search sessions are transformed to training samples like $(s,a,r,s')$, which are output continuously in the form of data stream and utilized by our algorithm to update the policy parameters in the learning component. Whenever the policy model is updated, it will be rewritten to the online KV system. Note that the two working loops in our system work in parallel but asynchronously, because the user log data generated in any search session cannot be utilized immediately.

The linear ranking mode used in our simulation is also adopted in this \textit{TaoBao} application. The ranking action of the search engine is a $27$-dim weight vector. The state of the environment is represented by a $90$-dim feature vector, which contains the item page features, user features and query features of the current search session. We add user and query information to the state feature since the ranking service in \textit{TaoBao} is for any type of users and there is no limitation on the input queries. We still adopt neural networks as the policy and value function approximators. However, to guarantee the online realtime performance and quick processing of the training data, the actor and critic networks have much smaller scale than those used in our simulation, with only $80$ and $64$ units in each of their two fully connected hidden layers, respectively. We implement DDPG and DDPG-FBE algorithms in our system and conduct one-week A/B test to compare the two algorithms. In each day of the test, the DDPG-FBE algorithm can lead to $2.7\% \sim 4.3\%$ more transaction amount than the DDPG algorithm \footnote{We cannot report the accurate transaction amount due to the information protection rule of Alibaba. Here we provide a reference index: the GMV achieved by Alibaba's China retail marketplace platforms surpassed $476$ billion U.S. dollars in the fiscal year of 2016 \cite{website:alizila}.}. The DDPG-FBE algorithm was also used for online ranking service on the TMall Double $11$ Global Shopping Festival of $2016$. Compared with the baseline algorithm (an LTR algorithm trained offline), our algorithm achieved more than $30\%$ growth in GMV at the end of that day.

\section{Conclusions}\label{Sec:7} 
In this paper, we propose to use reinforcement learning (RL) for ranking control in E-commerce searching scenarios. Our contributions are as follows. Firstly, we formally define the concept of search session Markov decision process (SSMDP) to formulate the multi-step ranking problem in E-commerce searching scenarios. Secondly, we analyze the property of SSMDP and theoretically prove the necessity of maximizing accumulative rewards. Lastly, we propose a novel policy gradient algorithm for learning an optimal ranking policy in an SSMDP. Experimental results in simulation and \textit{TaoBao} search engine show that our algorithm perform much better than the state-of-the-art LTR methods in the multi-step ranking problem, with more than $40\%$ and $30\%$ growth in gross merchandise volume, respectively.

\begin{acks}
We would like to thank our colleague − Yusen Zhan for useful discussions and supports of this work. We would also like to thank the anonymous referees for their valuable comments and helpful suggestions. Yang Yu is supported by Jiangsu SF (BK20160066).
\end{acks}

\bibliographystyle{ACM-Reference-Format}
\balance
\bibliography{kdd2018}


\begin{thebibliography}{33}


\ifx \showCODEN    \undefined \def \showCODEN     #1{\unskip}     \fi
\ifx \showDOI      \undefined \def \showDOI       #1{#1}\fi
\ifx \showISBNx    \undefined \def \showISBNx     #1{\unskip}     \fi
\ifx \showISBNxiii \undefined \def \showISBNxiii  #1{\unskip}     \fi
\ifx \showISSN     \undefined \def \showISSN      #1{\unskip}     \fi
\ifx \showLCCN     \undefined \def \showLCCN      #1{\unskip}     \fi
\ifx \shownote     \undefined \def \shownote      #1{#1}          \fi
\ifx \showarticletitle \undefined \def \showarticletitle #1{#1}   \fi
\ifx \showURL      \undefined \def \showURL       {\relax}        \fi
\providecommand\bibfield[2]{#2}
\providecommand\bibinfo[2]{#2}
\providecommand\natexlab[1]{#1}
\providecommand\showeprint[2][]{arXiv:#2}

\bibitem[\protect\citeauthoryear{Alizila}{Alizila}{2017}]%
        {website:alizila}
\bibfield{author}{\bibinfo{person}{Alizila}.} \bibinfo{year}{2017}\natexlab{}.
\newblock \bibinfo{title}{Joe Tsai Looks Beyond Alibaba’s RMB 3 Trillion
  Milestone}.
\newblock
  \bibinfo{howpublished}{\url{http://www.alizila.com/joe-tsai-beyond-alibabas-3-trillion-milestone/}}.
    (\bibinfo{year}{2017}).
\newblock


\bibitem[\protect\citeauthoryear{Auer}{Auer}{2002}]%
        {auer2002using}
\bibfield{author}{\bibinfo{person}{Peter Auer}.}
  \bibinfo{year}{2002}\natexlab{}.
\newblock \showarticletitle{Using confidence bounds for
  exploitation-exploration trade-offs}.
\newblock \bibinfo{journal}{\emph{Journal of Machine Learning Research}}
  \bibinfo{volume}{3}, \bibinfo{number}{Nov} (\bibinfo{year}{2002}),
  \bibinfo{pages}{397--422}.
\newblock


\bibitem[\protect\citeauthoryear{Brafman and Tennenholtz}{Brafman and
  Tennenholtz}{2002}]%
        {BrafmanT02}
\bibfield{author}{\bibinfo{person}{Ronen~I. Brafman} {and}
  \bibinfo{person}{Moshe Tennenholtz}.} \bibinfo{year}{2002}\natexlab{}.
\newblock \showarticletitle{{R-MAX} - {A} General Polynomial Time Algorithm for
  Near-Optimal Reinforcement Learning}.
\newblock \bibinfo{journal}{\emph{Journal of Machine Learning Research}}
  \bibinfo{volume}{3} (\bibinfo{year}{2002}), \bibinfo{pages}{213--231}.
\newblock


\bibitem[\protect\citeauthoryear{Burges, Shaked, Renshaw, Lazier, Deeds,
  Hamilton, and Hullender}{Burges et~al\mbox{.}}{2005}]%
        {burges2005learning}
\bibfield{author}{\bibinfo{person}{Chris Burges}, \bibinfo{person}{Tal Shaked},
  \bibinfo{person}{Erin Renshaw}, \bibinfo{person}{Ari Lazier},
  \bibinfo{person}{Matt Deeds}, \bibinfo{person}{Nicole Hamilton}, {and}
  \bibinfo{person}{Greg Hullender}.} \bibinfo{year}{2005}\natexlab{}.
\newblock \showarticletitle{Learning to rank using gradient descent}. In
  \bibinfo{booktitle}{\emph{{Proceedings of the 22nd International Conference
  on Machine Learning}}}. \bibinfo{pages}{89--96}.
\newblock


\bibitem[\protect\citeauthoryear{Cao, Xu, Liu, Li, Huang, and Hon}{Cao
  et~al\mbox{.}}{2006}]%
        {cao2006adapting}
\bibfield{author}{\bibinfo{person}{Yunbo Cao}, \bibinfo{person}{Jun Xu},
  \bibinfo{person}{Tie-Yan Liu}, \bibinfo{person}{Hang Li},
  \bibinfo{person}{Yalou Huang}, {and} \bibinfo{person}{Hsiao-Wuen Hon}.}
  \bibinfo{year}{2006}\natexlab{}.
\newblock \showarticletitle{Adapting ranking SVM to document retrieval}. In
  \bibinfo{booktitle}{\emph{{Proceedings of the 29th Annual International
  Conference on Research and Development in Information Retrieval
  (SIGIR'06)}}}. \bibinfo{pages}{186--193}.
\newblock


\bibitem[\protect\citeauthoryear{Cao, Qin, Liu, Tsai, and Li}{Cao
  et~al\mbox{.}}{2007}]%
        {cao2007learning}
\bibfield{author}{\bibinfo{person}{Zhe Cao}, \bibinfo{person}{Tao Qin},
  \bibinfo{person}{Tie-Yan Liu}, \bibinfo{person}{Ming-Feng Tsai}, {and}
  \bibinfo{person}{Hang Li}.} \bibinfo{year}{2007}\natexlab{}.
\newblock \showarticletitle{Learning to rank: from pairwise approach to
  listwise approach}. In \bibinfo{booktitle}{\emph{{Proceedings of the 24th
  International Conference on Machine Learning (ICML'07)}}}. ACM,
  \bibinfo{pages}{129--136}.
\newblock


\bibitem[\protect\citeauthoryear{Hofmann, Whiteson, and de~Rijke}{Hofmann
  et~al\mbox{.}}{2013}]%
        {hofmann2013balancing}
\bibfield{author}{\bibinfo{person}{Katja Hofmann}, \bibinfo{person}{Shimon
  Whiteson}, {and} \bibinfo{person}{Maarten de Rijke}.}
  \bibinfo{year}{2013}\natexlab{}.
\newblock \showarticletitle{Balancing exploration and exploitation in listwise
  and pairwise online learning to rank for information retrieval}.
\newblock \bibinfo{journal}{\emph{Information Retrieval}} \bibinfo{volume}{16},
  \bibinfo{number}{1} (\bibinfo{year}{2013}), \bibinfo{pages}{63--90}.
\newblock


\bibitem[\protect\citeauthoryear{Joachims}{Joachims}{2002}]%
        {joachims2002optimizing}
\bibfield{author}{\bibinfo{person}{Thorsten Joachims}.}
  \bibinfo{year}{2002}\natexlab{}.
\newblock \showarticletitle{Optimizing search engines using clickthrough data}.
  In \bibinfo{booktitle}{\emph{{Proceedings of the eighth ACM SIGKDD
  International Conference on Knowledge Discovery and Data Mining (KDD'02)}}}.
  ACM, \bibinfo{pages}{133--142}.
\newblock


\bibitem[\protect\citeauthoryear{Katariya, Kveton, Szepesvari, Vernade, and
  Wen}{Katariya et~al\mbox{.}}{2017}]%
        {katariya2017stochastic}
\bibfield{author}{\bibinfo{person}{Sumeet Katariya}, \bibinfo{person}{Branislav
  Kveton}, \bibinfo{person}{Csaba Szepesvari}, \bibinfo{person}{Claire
  Vernade}, {and} \bibinfo{person}{Zheng Wen}.}
  \bibinfo{year}{2017}\natexlab{}.
\newblock \showarticletitle{Stochastic Rank-1 Bandits}. In
  \bibinfo{booktitle}{\emph{Artificial Intelligence and Statistics}}.
  \bibinfo{pages}{392--401}.
\newblock


\bibitem[\protect\citeauthoryear{Kearns and Singh}{Kearns and Singh}{2002}]%
        {kearns2002near}
\bibfield{author}{\bibinfo{person}{Michael Kearns} {and}
  \bibinfo{person}{Satinder Singh}.} \bibinfo{year}{2002}\natexlab{}.
\newblock \showarticletitle{Near-optimal reinforcement learning in polynomial
  time}.
\newblock \bibinfo{journal}{\emph{Machine Learning}} \bibinfo{volume}{49},
  \bibinfo{number}{2-3} (\bibinfo{year}{2002}), \bibinfo{pages}{209--232}.
\newblock


\bibitem[\protect\citeauthoryear{Kveton, Szepesvari, Wen, and Ashkan}{Kveton
  et~al\mbox{.}}{2015a}]%
        {kveton2015cascading}
\bibfield{author}{\bibinfo{person}{Branislav Kveton}, \bibinfo{person}{Csaba
  Szepesvari}, \bibinfo{person}{Zheng Wen}, {and} \bibinfo{person}{Azin
  Ashkan}.} \bibinfo{year}{2015}\natexlab{a}.
\newblock \showarticletitle{Cascading bandits: Learning to rank in the cascade
  model}. In \bibinfo{booktitle}{\emph{Proceedings of the 32nd International
  Conference on Machine Learning (ICML-15)}}. \bibinfo{pages}{767--776}.
\newblock


\bibitem[\protect\citeauthoryear{Kveton, Wen, Ashkan, and Szepesvari}{Kveton
  et~al\mbox{.}}{2015b}]%
        {kveton2015combinatorial}
\bibfield{author}{\bibinfo{person}{Branislav Kveton}, \bibinfo{person}{Zheng
  Wen}, \bibinfo{person}{Azin Ashkan}, {and} \bibinfo{person}{Csaba
  Szepesvari}.} \bibinfo{year}{2015}\natexlab{b}.
\newblock \showarticletitle{Combinatorial cascading bandits}. In
  \bibinfo{booktitle}{\emph{{Advances in Neural Information Processing Systems
  (NIPS'15)}}}. \bibinfo{pages}{1450--1458}.
\newblock


\bibitem[\protect\citeauthoryear{Lagr{\'e}e, Vernade, and Cappe}{Lagr{\'e}e
  et~al\mbox{.}}{2016}]%
        {lagree2016multiple}
\bibfield{author}{\bibinfo{person}{Paul Lagr{\'e}e}, \bibinfo{person}{Claire
  Vernade}, {and} \bibinfo{person}{Olivier Cappe}.}
  \bibinfo{year}{2016}\natexlab{}.
\newblock \showarticletitle{Multiple-play bandits in the position-based model}.
  In \bibinfo{booktitle}{\emph{{Advances in Neural Information Processing
  Systems (NIPS'16)}}}. \bibinfo{pages}{1597--1605}.
\newblock


\bibitem[\protect\citeauthoryear{Langford and Zhang}{Langford and
  Zhang}{2008}]%
        {langford2008epoch}
\bibfield{author}{\bibinfo{person}{John Langford} {and} \bibinfo{person}{Tong
  Zhang}.} \bibinfo{year}{2008}\natexlab{}.
\newblock \showarticletitle{The epoch-greedy algorithm for multi-armed bandits
  with side information}. In \bibinfo{booktitle}{\emph{Advances in neural
  information processing systems}}. \bibinfo{pages}{817--824}.
\newblock


\bibitem[\protect\citeauthoryear{Li, Wu, and Burges}{Li et~al\mbox{.}}{2008}]%
        {li2008mcrank}
\bibfield{author}{\bibinfo{person}{Ping Li}, \bibinfo{person}{Qiang Wu}, {and}
  \bibinfo{person}{Christopher~J Burges}.} \bibinfo{year}{2008}\natexlab{}.
\newblock \showarticletitle{Mcrank: Learning to rank using multiple
  classification and gradient boosting}. In \bibinfo{booktitle}{\emph{{Advances
  in Neural Information Processing Systems (NIPS'08)}}}.
  \bibinfo{pages}{897--904}.
\newblock


\bibitem[\protect\citeauthoryear{Li, Wang, Zhang, and Chen}{Li
  et~al\mbox{.}}{2016}]%
        {li2016contextual}
\bibfield{author}{\bibinfo{person}{Shuai Li}, \bibinfo{person}{Baoxiang Wang},
  \bibinfo{person}{Shengyu Zhang}, {and} \bibinfo{person}{Wei Chen}.}
  \bibinfo{year}{2016}\natexlab{}.
\newblock \showarticletitle{Contextual combinatorial cascading bandits}. In
  \bibinfo{booktitle}{\emph{{International Conference on Machine Learning
  (ICML'16)}}}. \bibinfo{pages}{1245--1253}.
\newblock


\bibitem[\protect\citeauthoryear{Lillicrap, Hunt, Pritzel, Heess, Erez, Tassa,
  Silver, and Wierstra}{Lillicrap et~al\mbox{.}}{2015}]%
        {lillicrap2015continuous}
\bibfield{author}{\bibinfo{person}{Timothy~P Lillicrap},
  \bibinfo{person}{Jonathan~J Hunt}, \bibinfo{person}{Alexander Pritzel},
  \bibinfo{person}{Nicolas Heess}, \bibinfo{person}{Tom Erez},
  \bibinfo{person}{Yuval Tassa}, \bibinfo{person}{David Silver}, {and}
  \bibinfo{person}{Daan Wierstra}.} \bibinfo{year}{2015}\natexlab{}.
\newblock \showarticletitle{Continuous control with deep reinforcement
  learning}.
\newblock \bibinfo{journal}{\emph{arXiv preprint arXiv:1509.02971}}
  (\bibinfo{year}{2015}).
\newblock


\bibitem[\protect\citeauthoryear{Liu et~al\mbox{.}}{Liu et~al\mbox{.}}{2009}]%
        {liu2009learning}
\bibfield{author}{\bibinfo{person}{Tie-Yan Liu} {et~al\mbox{.}}}
  \bibinfo{year}{2009}\natexlab{}.
\newblock \showarticletitle{Learning to rank for information retrieval}.
\newblock \bibinfo{journal}{\emph{Foundations and Trends{\textregistered} in
  Information Retrieval}} \bibinfo{volume}{3}, \bibinfo{number}{3}
  (\bibinfo{year}{2009}), \bibinfo{pages}{225--331}.
\newblock


\bibitem[\protect\citeauthoryear{Maei, Szepesv{\'{a}}ri, Bhatnagar, and
  Sutton}{Maei et~al\mbox{.}}{2010}]%
        {MaeiSBS10}
\bibfield{author}{\bibinfo{person}{Hamid~R. Maei}, \bibinfo{person}{Csaba
  Szepesv{\'{a}}ri}, \bibinfo{person}{Shalabh Bhatnagar}, {and}
  \bibinfo{person}{Richard~S. Sutton}.} \bibinfo{year}{2010}\natexlab{}.
\newblock \showarticletitle{Toward off-policy learning control with function
  approximation}. In \bibinfo{booktitle}{\emph{{Proceedings of the 27th
  International Conference on Machine Learning}}}. \bibinfo{pages}{719--726}.
\newblock


\bibitem[\protect\citeauthoryear{Mnih, Kavukcuoglu, Silver, Rusu, Veness,
  Bellemare, Graves, Riedmiller, Fidjeland, Ostrovski, et~al\mbox{.}}{Mnih
  et~al\mbox{.}}{2015}]%
        {mnih2015human}
\bibfield{author}{\bibinfo{person}{Volodymyr Mnih}, \bibinfo{person}{Koray
  Kavukcuoglu}, \bibinfo{person}{David Silver}, \bibinfo{person}{Andrei~A
  Rusu}, \bibinfo{person}{Joel Veness}, \bibinfo{person}{Marc~G Bellemare},
  \bibinfo{person}{Alex Graves}, \bibinfo{person}{Martin Riedmiller},
  \bibinfo{person}{Andreas~K Fidjeland}, \bibinfo{person}{Georg Ostrovski},
  {et~al\mbox{.}}} \bibinfo{year}{2015}\natexlab{}.
\newblock \showarticletitle{Human-level control through deep reinforcement
  learning}.
\newblock \bibinfo{journal}{\emph{Nature}} \bibinfo{volume}{518},
  \bibinfo{number}{7540} (\bibinfo{year}{2015}), \bibinfo{pages}{529--533}.
\newblock


\bibitem[\protect\citeauthoryear{Nallapati}{Nallapati}{2004}]%
        {nallapati2004discriminative}
\bibfield{author}{\bibinfo{person}{Ramesh Nallapati}.}
  \bibinfo{year}{2004}\natexlab{}.
\newblock \showarticletitle{Discriminative models for information retrieval}.
  In \bibinfo{booktitle}{\emph{{Proceedings of the 27th Annual International
  ACM SIGIR Conference on Research and Development in Information Retrieval
  (SIGIR'04)}}}. ACM, \bibinfo{pages}{64--71}.
\newblock


\bibitem[\protect\citeauthoryear{Radlinski, Kleinberg, and Joachims}{Radlinski
  et~al\mbox{.}}{2008}]%
        {radlinski2008learning}
\bibfield{author}{\bibinfo{person}{Filip Radlinski}, \bibinfo{person}{Robert
  Kleinberg}, {and} \bibinfo{person}{Thorsten Joachims}.}
  \bibinfo{year}{2008}\natexlab{}.
\newblock \showarticletitle{Learning diverse rankings with multi-armed
  bandits}. In \bibinfo{booktitle}{\emph{Proceedings of the 25th international
  conference on Machine learning}}. ACM, \bibinfo{pages}{784--791}.
\newblock


\bibitem[\protect\citeauthoryear{Schulman, Levine, Abbeel, Jordan, and
  Moritz}{Schulman et~al\mbox{.}}{2015}]%
        {schulman2015trust}
\bibfield{author}{\bibinfo{person}{John Schulman}, \bibinfo{person}{Sergey
  Levine}, \bibinfo{person}{Pieter Abbeel}, \bibinfo{person}{Michael Jordan},
  {and} \bibinfo{person}{Philipp Moritz}.} \bibinfo{year}{2015}\natexlab{}.
\newblock \showarticletitle{Trust region policy optimization}. In
  \bibinfo{booktitle}{\emph{{Proceedings of the 32nd International Conference
  on Machine Learning (ICML'15)}}}. \bibinfo{pages}{1889--1897}.
\newblock


\bibitem[\protect\citeauthoryear{Silver, Huang, Maddison, Guez, Sifre, Van
  Den~Driessche, Schrittwieser, Antonoglou, Panneershelvam, Lanctot,
  et~al\mbox{.}}{Silver et~al\mbox{.}}{2016}]%
        {silver2016mastering}
\bibfield{author}{\bibinfo{person}{David Silver}, \bibinfo{person}{Aja Huang},
  \bibinfo{person}{Chris~J Maddison}, \bibinfo{person}{Arthur Guez},
  \bibinfo{person}{Laurent Sifre}, \bibinfo{person}{George Van Den~Driessche},
  \bibinfo{person}{Julian Schrittwieser}, \bibinfo{person}{Ioannis Antonoglou},
  \bibinfo{person}{Veda Panneershelvam}, \bibinfo{person}{Marc Lanctot},
  {et~al\mbox{.}}} \bibinfo{year}{2016}\natexlab{}.
\newblock \showarticletitle{Mastering the game of Go with deep neural networks
  and tree search}.
\newblock \bibinfo{journal}{\emph{Nature}} \bibinfo{volume}{529},
  \bibinfo{number}{7587} (\bibinfo{year}{2016}), \bibinfo{pages}{484--489}.
\newblock


\bibitem[\protect\citeauthoryear{Silver, Lever, Heess, Degris, Wierstra, and
  Riedmiller}{Silver et~al\mbox{.}}{2014}]%
        {silver2014deterministic}
\bibfield{author}{\bibinfo{person}{David Silver}, \bibinfo{person}{Guy Lever},
  \bibinfo{person}{Nicolas Heess}, \bibinfo{person}{Thomas Degris},
  \bibinfo{person}{Daan Wierstra}, {and} \bibinfo{person}{Martin Riedmiller}.}
  \bibinfo{year}{2014}\natexlab{}.
\newblock \showarticletitle{Deterministic policy gradient algorithms}. In
  \bibinfo{booktitle}{\emph{{Proceedings of the 31st International Conference
  on Machine Learning (ICML'14)}}}. \bibinfo{pages}{387--395}.
\newblock


\bibitem[\protect\citeauthoryear{Slivkins, Radlinski, and Gollapudi}{Slivkins
  et~al\mbox{.}}{2013}]%
        {slivkins2013ranked}
\bibfield{author}{\bibinfo{person}{Aleksandrs Slivkins}, \bibinfo{person}{Filip
  Radlinski}, {and} \bibinfo{person}{Sreenivas Gollapudi}.}
  \bibinfo{year}{2013}\natexlab{}.
\newblock \showarticletitle{Ranked bandits in metric spaces: learning diverse
  rankings over large document collections}.
\newblock \bibinfo{journal}{\emph{Journal of Machine Learning Research}}
  \bibinfo{volume}{14}, \bibinfo{number}{Feb} (\bibinfo{year}{2013}),
  \bibinfo{pages}{399--436}.
\newblock


\bibitem[\protect\citeauthoryear{Sutton and Barto}{Sutton and Barto}{1998}]%
        {sutton1998reinforcement}
\bibfield{author}{\bibinfo{person}{R.S. Sutton} {and} \bibinfo{person}{A.G.
  Barto}.} \bibinfo{year}{1998}\natexlab{}.
\newblock \bibinfo{booktitle}{\emph{Reinforcement {Learning}: An
  {Introduction}}}.
\newblock \bibinfo{publisher}{MIT Press}.
\newblock


\bibitem[\protect\citeauthoryear{Sutton, McAllester, Singh, and Mansour}{Sutton
  et~al\mbox{.}}{2000}]%
        {sutton2000policy}
\bibfield{author}{\bibinfo{person}{Richard~S Sutton}, \bibinfo{person}{David~A
  McAllester}, \bibinfo{person}{Satinder~P Singh}, {and}
  \bibinfo{person}{Yishay Mansour}.} \bibinfo{year}{2000}\natexlab{}.
\newblock \showarticletitle{Policy gradient methods for reinforcement learning
  with function approximation}. In \bibinfo{booktitle}{\emph{{Advances in
  Neural Information Processing Systems (NIPS'00)}}}.
  \bibinfo{pages}{1057--1063}.
\newblock


\bibitem[\protect\citeauthoryear{Watkins}{Watkins}{1989}]%
        {watkins1989learning}
\bibfield{author}{\bibinfo{person}{C.J.C.H. Watkins}.}
  \bibinfo{year}{1989}\natexlab{}.
\newblock \emph{\bibinfo{title}{Learning from delayed rewards}}.
\newblock \bibinfo{thesistype}{Ph.D. Dissertation}. \bibinfo{school}{King's
  College, Cambridge}.
\newblock


\bibitem[\protect\citeauthoryear{Williams}{Williams}{1992}]%
        {williams1992simple}
\bibfield{author}{\bibinfo{person}{Ronald~J Williams}.}
  \bibinfo{year}{1992}\natexlab{}.
\newblock \showarticletitle{Simple statistical gradient-following algorithms
  for connectionist reinforcement learning}.
\newblock \bibinfo{journal}{\emph{Machine learning}} \bibinfo{volume}{8},
  \bibinfo{number}{3-4} (\bibinfo{year}{1992}), \bibinfo{pages}{229--256}.
\newblock


\bibitem[\protect\citeauthoryear{Yue and Joachims}{Yue and Joachims}{2009}]%
        {yue2009interactively}
\bibfield{author}{\bibinfo{person}{Yisong Yue} {and} \bibinfo{person}{Thorsten
  Joachims}.} \bibinfo{year}{2009}\natexlab{}.
\newblock \showarticletitle{Interactively optimizing information retrieval
  systems as a dueling bandits problem}. In
  \bibinfo{booktitle}{\emph{{Proceedings of the 26th Annual International
  Conference on Machine Learning (ICML'09)}}}. ACM,
  \bibinfo{pages}{1201--1208}.
\newblock


\bibitem[\protect\citeauthoryear{Zoghi, Tunys, Ghavamzadeh, Kveton, Szepesvari,
  and Wen}{Zoghi et~al\mbox{.}}{2017}]%
        {zoghi2017online}
\bibfield{author}{\bibinfo{person}{Masrour Zoghi}, \bibinfo{person}{Tomas
  Tunys}, \bibinfo{person}{Mohammad Ghavamzadeh}, \bibinfo{person}{Branislav
  Kveton}, \bibinfo{person}{Csaba Szepesvari}, {and} \bibinfo{person}{Zheng
  Wen}.} \bibinfo{year}{2017}\natexlab{}.
\newblock \showarticletitle{Online Learning to Rank in Stochastic Click
  Models}. In \bibinfo{booktitle}{\emph{International Conference on Machine
  Learning}}. \bibinfo{pages}{4199--4208}.
\newblock


\bibitem[\protect\citeauthoryear{Zong, Ni, Sung, Ke, Wen, and Kveton}{Zong
  et~al\mbox{.}}{2016}]%
        {zong2016cascading}
\bibfield{author}{\bibinfo{person}{Shi Zong}, \bibinfo{person}{Hao Ni},
  \bibinfo{person}{Kenny Sung}, \bibinfo{person}{Nan~Rosemary Ke},
  \bibinfo{person}{Zheng Wen}, {and} \bibinfo{person}{Branislav Kveton}.}
  \bibinfo{year}{2016}\natexlab{}.
\newblock \showarticletitle{Cascading bandits for large-scale recommendation
  problems}. In \bibinfo{booktitle}{\emph{{Proceedings of the Thirty-Second
  Conference on Uncertainty in Artificial Intelligence (UAI'16)}}}.
  \bibinfo{pages}{835--844}.
\newblock


\end{thebibliography}

\end{document}